\documentclass[twoside]{article}

\usepackage[accepted]{aistats2025}
\usepackage[round]{natbib}

\usepackage{caption}
\usepackage{wrapfig}
\usepackage[utf8]{inputenc} 
\usepackage{multicol}
\usepackage[T1]{fontenc}    
\usepackage{hyperref}       
\usepackage{url}            
\usepackage{booktabs}       
\usepackage{amsfonts}       
\usepackage{nicefrac}       
\usepackage{microtype}      
\usepackage{xcolor}         
\usepackage{bm}
\usepackage{algorithm}
\usepackage[noend]{algorithmic}
\usepackage{paralist}
\usepackage{enumitem}

\usepackage{amsmath}
\usepackage{amssymb}
\usepackage{mathtools}
\usepackage{amsthm}
\usepackage{amsfonts}

\DeclareMathOperator*{\Exp}{\mathbb{E}}

\usepackage[capitalize,noabbrev]{cleveref}

\theoremstyle{plain}
\newtheorem{theorem}{Theorem}[section]
\newtheorem{prop}[theorem]{Proposition}
\newtheorem{lemma}[theorem]{Lemma}
\newtheorem{coro}[theorem]{Corollary}
\theoremstyle{definition}

\newtheorem{assumption}[theorem]{Assumption}
\theoremstyle{remark}

\begin{document}

\twocolumn[

\runningauthor{Mazzetto, Esfandiarpoor, Singirikonda, Upfal, Bach }

\runningtitle{An Adaptive Method for Weak Supervision with Drifting Data}

\aistatstitle{An Adaptive Method for \\ Weak Supervision with Drifting Data}

\aistatsauthor{ Alessio Mazzetto \And Reza Esfandiarpoor \And Akash Singirikonda  \And Eli Upfal
}

\aistatsauthor{Stephen Bach}

\aistatsaddress{ \\Brown University }]

\begin{abstract}
We introduce an adaptive method with formal quality guarantees for weak supervision in a non-stationary setting. Our goal is to infer the unknown labels of a sequence of data by using weak supervision sources that provide independent noisy signals of the correct classification for each data point. This setting includes crowdsourcing and programmatic weak supervision. We focus on the non-stationary case, where the accuracy of the weak supervision sources can drift over time, e.g., because of changes in the underlying data distribution. Due to the drift, older data could provide misleading information to infer the label of the current data point. Previous work relied on a priori assumptions on the magnitude of the drift to decide how much data to use from the past. In contrast, 
our algorithm does not require any assumptions on the drift, and it adapts based on the input by dynamically varying its window size. In particular, at each step, our algorithm estimates the current accuracies of the weak supervision sources by identifying a window of past observations that guarantees a near-optimal minimization of the trade-off between the error due to the variance of the estimation and the error due to the drift. Experiments on synthetic and real-world labelers show that our approach adapts to the drift.
\end{abstract}

\section{Introduction}
In order to efficiently create training data for machine learning, programmatic weak supervision~\citep{ratner:neurips16, ratner2017snorkel, zhang:arxiv22} estimates the accuracy of multiple noisy sources of labels without access to ground truth.
Given a set of \emph{labeling functions} that vote on the true label for each unlabeled example, the goal is to infer the latent ground truth.
Once inferred, these labels can be used as training data.
In this paper, we study the \emph{non-stationary} setting, in which the accuracy of each labeling function can drift over time because of changes in the underlying data.
For example, in an image classification task, latent subclasses that make up each class might shift over time.
If the task is to classify animals into categories like ``mammal'' and ``bird,'' the accuracy of a weak labeler that looks for attributes like wings might change in accuracy if animals like bats become more or less prevalent.
We ask the question, ``Under what conditions can we detect changes in the accuracies of weak labelers over time and bound their error without access to ground truth?''

Programmatic weak supervision is important for creating training data sets when resources are limited.
It can be used for natural language processing~\citep{safranchik:aaai20,yu:naacl21,zhang:neurips21}, computer vision~\citep{varma:neurips17,chen:iccv19,fu2020fast}, tabular data~\citep{chatterjee:aaai20,arachie:jmlr21}, and other modalities~\citep{sala:neurips19,shin:iclr22}.
It has also enabled machine learning applications in industry~\citep{bach:sigmod19-industrial, bringer:deem19, suri:vldb20} and academia~\citep{callahan:npjdigmed19,fries:nc2021}.
Even when prompting or fine-tuning large pre-trained models, weak supervision can unlock improved quality and enable adaptation to new tasks~\citep{smith:arxiv22, arora:arxiv22,yu:acldemo23}.

The central modeling challenge in programmatic weak supervision is estimating the probabilistic relationships among the votes of the weak labelers and the latent ground truth.
It is hard because, without access to ground truth labels, the observed votes can be explained in many different ways.
Perhaps the votes tend to agree because they are all accurate labelers.
Or perhaps they are all inaccurate.
Perhaps there are correlations among the votes caused by relying on similar decision processes.
If one assumes that the votes are conditionally independent given the true label and that the examples are \emph{independent}, and \emph{identically distributed} (i.i.d.), this is equivalent to
 the Dawid-Skene model~\citep{dawid1979maximum} that is the basis for many related works in crowdsourcing~\citep{raykar:jmlr10,liu:neurips12,parisi:pnas14,joglekar:icde15,zhang:jmlr16}.
Many works on crowdsourcing and weak supervision have relaxed the conditional independence assumption in various ways to account for a wide range of weak labelers~\citep{balasubramani:colt15,balasubramani:neurips15,balasubramani:neurips16,bach:icml17, varma:icml19, arachie:aaai19, mazzetto:aistats21, mazzetto:icml21, arachie:jmlr21}.

With two exceptions discussed below, all these aforementioned works assume that the examples are drawn independently from a fixed distribution. This is a restrictive assumption when data is collected over time, and it is natural to observe a change, or \emph{drift}, in the distribution of the examples. In our work, we relax the identically distributed assumption, and assume only that the examples are independent. This introduces a trade-off: to obtain a good estimate at the current time, using more past examples provides more data, which might result in a better estimate if that data is similarly distributed, but might harm the estimate if the window includes a significant distribution drift.

Much prior work has addressed the problem of drifting data in the supervised learning setting~\citep{gama:acmsurvey14,lu:tkde18}.
These methods generally rely on labeled data that is unavailable in the weakly supervised setting.
Another broad line of work has viewed drift detection as an unsupervised problem, looking for non-stationarity in arbitrary distributions~\citep{barry:jasa93,killick:jasa12,truong:sigproc20}.
These methods generally assume a prior distribution on the locations in time of drift.
That prior can be either defined explicitly in a Bayesian framework or implicitly via a heuristic cost function that penalizes the trade-off between better fitting the data and finding more drift points.
In a similar vein, previous works on relaxing the i.i.d.\ assumption with multiple noisy labelers have placed assumptions on how much their accuracies can drift~\citep{bonald2016streaming,fu2020fast}.
In contrast, our goal is to estimate the labelers' accuracies, without prior assumption on the drift.
The lack of any assumptions means that each individual sample can come from its own distribution.
 This is a very challenging problem, as the drift is unknown and we cannot estimate the drift from the data, since we have access to only a single sample from each distribution.

\textbf{Our Contributions.}
We introduce the first \emph{adaptive} algorithm for programmatic weak supervision in the presence of drift with formal guarantees on the quality of its parameter estimates.
The advantage of an adaptive algorithm is that it can react in a rigorously principled way to changes in the accuracies of the weak labelers as they occur (as opposed to having to make an assumption on how much drift will occur).
When the underlying process is stationary, it can accumulate as much data as possible by using a large window of time in order to best estimate the accuracies of the labelers.
When drift does occur, it can minimize the drift error by using a smaller window with the most recent (and most relevant) data to estimate the accuracies.

Our method selects the amount of data to use based on differences in the rates of agreement among the labelers.
We derive a principled decision rule for this selection and provide a rigorous analysis that bounds the resulting error of the estimated accuracies of the labelers. 
Our novel bound separates the statistical error of estimating the parameters from the error caused by possible drift.
This analysis enables the algorithm to select a close-to-optimal trade-off to minimize the worst-case error. 

The conceptual difference between our approach and all previous work on weak supervision with drifting data is that we do not rely on prior information about the drift, or try to learn the drift from the data (both unrealistic in many applications). Instead, at each time step, our algorithm compares its estimation obtained using different window sizes and uses this information to detect drift and adjust the window size for the decision at that step. We analytically prove that this information is sufficient to allow the algorithm to efficiently adapt to drift in distribution, without explicitly estimating the magnitude of the drift.


We demonstrate the functionality and the advantage
of our algorithm over fixed-window-size strategies in
several experimental settings, including synthetic data,
image recognition, and video classification. The results show
 that our algorithm adapts to the drift as it occurs, dynamically selecting the amount of data (window size) to use in an effective way. Unlike fixed-window-size strategies, our approach consistently maintains high accuracy as it adapts to the changing drift.



\section{Problem Statement}
Given a vector $\bm{v} \in \mathbb{R}^q$, let $\lVert \bm{v} \rVert_{\infty} = \max_{1 \leq i \leq q}|v_i|$. Similarly, given a matrix $\bm{C} \in \mathbb{R}^{q\times q}$, we define $\lVert \bm{C} \rVert_{\infty} = \max_{i,j}|C_{ij}|$. A binary classification task is specified by a function $y:\mathcal{X} \mapsto \mathcal{Y}$, where $\mathcal{X} = \mathbb{R}^d$ is the classification domain and $\mathcal{Y} = \{ -1,1\}$ is the label space.  Given $x$, we would like to infer its label $y(x)$.
We assume access to $n$ weak labeling functions $\ell_1,\ldots,\ell_n$, where each $\ell_i : \mathcal{X} \rightarrow \mathcal{Y}$ provides a tentative labeling of the item $x$. For example, each weak labeling function can be a classifier that was trained for a related task, or a decision rule based on a simple programmatic criterion. The weak labeling functions $\ell_1(x), \ldots, \ell_n(x)$ are the only information sources for the labels in our classification task.

We receive a sequence of examples $X_1, X_2, \ldots $ over time.
For any given time $t$, our goal is to obtain an accurate estimate of the correct label $y(X_t)$ of $X_t$ given the weak labelling functions $\ell_1,\ldots,\ell_n$ and the input sequence up to time $t$, $X_1,\ldots,X_t$. 

We adapt the standard assumptions used in analyzing weak supervision, in particular crowdsourcing, with no drift~\citep{dawid1979maximum,ratner2017snorkel}. We first assume that the input sequence $( X_t)_{t \in \mathbb{N}}$ is an independent, but not identically distributed stochastic process. Any finite subset of its random variables are mutually independent, and each $X_t$ is sampled from a distribution $D_{t}$ over $\mathcal{X}$ that can drift over time. Formally,  this is stated with the following assumption.
\begin{assumption}
\label{assu:sample-independents}
For any finite $t \geq 1$, the input vector 
$(X_1,\dots,X_t)$ is distributed as $\prod_{i=1}^t D_{i}$. 
\end{assumption}
The second assumption is that the weak labelers have independent errors conditioned on the true label. 
\begin{assumption}
\label{assu:errors-independents}
For any $t \geq 1$ and $i\neq j$, 
for $X_t \sim \mathcal{D}_t$, we have that the events 
$\{ \ell_i(X_t) \neq y(X_t) \}$ and $ \{ \ell_j(X_t) \neq y(X_t) \}$ are independent given $y(X_t)$.
\end{assumption} 
This conditional independence is widely adopted across domains such as programmatic weak supervision and crowdsourcing \citep{dawid1979maximum, liu:neurips12, parisi:pnas14, joglekar:icde15, zhang:jmlr16}, and methods leveraging this assumption consistently demonstrate robust performance in practical applications \citep{ratner2017snorkel}. Recent research has explored alternative, weaker assumptions, such as access to a limited amount of labeled data \citep{arachie:aaai19, arachie:jmlr21, mazzetto:aistats21, mazzetto:icml21}, or parametric assumptions on the joint distributions of the true labels and outputs from weak labelers \citep{fu2020fast, ratner2019training}.
In Appendix~\ref{app:relaxation}, we detail how our method can be extended to incorporate one of these alternative assumptions.


We let the accuracy of the weak labeler $i$ at time $t$ be
\begin{align}
\label{eq:p-def}
    p_i(t) \doteq  \Pr_{X \sim D_{t}}\left( \ell_i(X) = y(X) \right) \enspace .
\end{align}
The value $p_i(t) \in [0,1]$ is the probability that the weak labeler $\ell_i$ is correct with a sample $X_t \sim D_t$. The accuracy probability $p_i(t)$ is a function of the input distribution $D_{t}$ and therefore may drift in time. We let $\bm{p}(t) = ( p_1(t), \ldots, p_n(t))$.


\textbf{Example.} Assume that the classification task is to distinguish whether an input image contains a cat or a dog.  Let $\ell_{\mathrm{tail}}$ be a weak labeler that detects whether an animal has a tail or not. This weak labeler provides no signal if we only observe images of cats and dogs that both have tails, however, the relevance of this classifier can change over time: if the probability of observing a dog without a tail (e.g., a bulldog)  grows over time, this weak labeler can provide a stronger signal towards the right classification.
{Our goal is to adapt dynamically to the change in accuracy of the weak labelers.} 

\textbf{Sources of Drift.}
For concreteness, we analyze our algorithm with respect to a drift in the input distribution over $\mathcal{X}$ (also referred to as \emph{covariate shift}). However, our analysis applies to a more general case, since it only relies on the variation of the accuracy of the weak labelers $p_i(t)$, and is agnostic to the underlying cause of this variation.

As an example, in addition to drift in the input distribution, we can also allow a change in the functionality of the labeling functions. 
For example, a human labeler can get tired and make more mistakes, or a sensor's accuracy can be affected by a change of light or temperature. 
Formally, instead of a labelling function $\ell_i(X)$ we have a family of labelling functions 
$\{ \ell_{i,t}(X)~|~t\geq 1\}$. Equation (\ref{eq:p-def}) is replaced with 
$
    p_i(t) \doteq  \Pr_{X \sim D_{t}}\left( \ell_{i,t}(X) = y(X) \right),$
and the algorithm and analysis are the same.

In a similar manner, our analysis can be extended to the setting where the binary classification task $y(x)$ also changes over time (referred to as \emph{concept drift}), by replacing $y(x)$ with a time-dependent function $y_t(x)$. This extension requires modifying \Cref{assu:errors-independents} and Equation~\eqref{eq:p-def} accordingly, but it does not change our algorithm or analysis.

\section{Related Work}
\label{sec:related}

To our knowledge, only two prior works have considered relaxations of the identical distribution (no drift) assumption in learning from multiple noisy sources of labels.
They both require assumptions on how much the accuracies of the labelers can change over time.
The first ~\citep{bonald2016streaming} assumes that the accuracy of the weak labelers can change at most by a constant at each step, i.e., there exists $\Delta > 0$, known a priori, such that $\lVert \bm{p}(t) - \bm{p}(t+1) \rVert_{\infty} \leq \Delta$ for all $t \geq 1$. 
The second~\citep{fu2020fast} assumes that the KL divergence between two consecutive distributions is upper bounded by a constant $\Delta$.
These are similar assumptions: an upper bound on the magnitude of the drift allows these methods to determine before execution how much information to use from the past.

These algorithms are impractical as the value  $\Delta$ is unknown in practice, and they cannot adapt to changes in the rate of drift over time.
If the algorithm overestimates the drift, then it will use a smaller amount of data than it should, resulting in a greater statistical error in its estimates of the labelers' accuracies.
If it underestimates the drift, then the algorithm will use too much data and incur a large error due to the drift.
In contrast, in this work, our goal is to dynamically choose the window size as a function of the observed votes without requiring any prior assumptions on the magnitude of the drift.
In other words, our approach is to \emph{adapt} to the drift as it occurs. The adaptivity to the drift is important to capture the changes in the data distributions. As we will see in \Cref{subsection:synthetic}, algorithms that rely on a fixed choice of $\Delta$ are severely limited when the rate of drift itself is changing. 

There is a vast literature that addresses the challenges of coping with non-stationary data in numerous different settings. 
In our work, we focus on a drift setting where we observe data from a non-stationary distribution, and we have access to a single sample at each time step. 
Within this drift setting, a relevant sequence of works \citep{bartlett1992learning,long1998complexity,mohri2012new,hanneke2019statistical}) has studied the supervised learning problem,
assuming some known upper bound on the drift.
The minimax error for density estimation with distribution drift was studied by \cite{mazzetto2023nonparametric}, again with some a priori assumption on the drift rate.
Recent work provides an adaptive algorithm for agnostic learning of a family of functions with distribution drift \citep{mazzetto2023adaptive}, and analogous results were proven for discrete distribution estimation \citep{mazzetto2023nonparametric,mazzetto2024improved}, model selection \cite{han2024model}, and vector quantization \citep{mazzetto2025center}.
While these works study similar drift settings, their results do not directly apply to our weak supervision setting.  Our work is the first to provide an adaptive algorithm for weak supervision and crowdsourcing in a non-stationary setting and without any prior assumptions on the drift.

A process that reveals just one sample per step is more natural but harder to handle than processes that reveal a set of inputs in each step \citep[e.g.,][]{bai2022adapting,zhang:arxiv22}. In the latter scenario, an algorithm has multiple samples from each distribution to learn the drift error \citep{mohri2012new,awasthi2023theory}.
We also note that the non-stationary setting was also extensively studied in reinforcement learning \citep[e.g.,][]{auer2019adaptively,chen2019new,wei2021non}.
That setting significantly differs from ours, as the goal is to minimize the regret, and the distribution of the samples is also affected by the decisions taken by a policy on the environment.

Finally, we would like to remark that there is a vast literature on the drift detection problem in both learning and data mining \citep[e.g.,][]{gama:acmsurvey14,lu:tkde18, agrahari2022concept,yu2022meta, li2022ddg, yu2023type}. Unlike these works, our goal is to provably decide over a window of data depending on the drift within that window.




\section{Preliminary Results}
\label{sec:preliminary}
Our work builds on the following results that study the problem in settings where the accuracy probabilities are known, or there is no drift in the input distribution.

Assume first that  the accuracy $\bm{p}(t) = (p_1(t),\ldots,p_n(t))$ of the weak labelers at any time $t\geq 1$ are known. With \Cref{assu:errors-independents}, it is known that the optimal aggregation rule for classifying $X_t$ is a weighted majority vote of $\ell_1(X_t),\ldots,\ell_n(X_t)$, where the weights are a function
of $\bm{p}(t)$ \citep{nitzan1982optimal}. 
In particular, consider the family of weighted majority classifiers  $f_{\bm{w}} : \mathcal{X} \mapsto \mathcal{Y}$ with weights $\bm{w} = (w_1,\ldots,w_n)$, i.e.,  $f_{\bm{w}}(x) = \mathrm{sign}\left( \sum_{i=1}^n w_i \ell_i(x) \right)$. The optimal aggregation of $\ell_1(X_t), \ldots, \ell_n(X_t)$ is given by $f_{\bm{w}^*(t)}$ where 
\begin{align}
\label{eq:w^*}
\bm{w}^*(t) = \left( \ln\left(\frac{p_1(t)}{1-p_1(t)}\right), \ldots,  \ln\left(\frac{p_n(t)}{1-p_n(t)}\right) \right) .
\end{align}
The above result implies that under \Cref{assu:errors-independents}, the knowledge of the weak labelers' accuracies is sufficient to obtain the optimal aggregation rule.
In weak supervision and crowdsourcing applications, the accuracy probabilities of the weak labelers are unknown.  
Several methods for estimating $\bm{p}(t)$ using previous samples, have been proposed
in the literature in a setting without distribution drift~\citep{dawid1979maximum,ghosh2011moderates,zhang:jmlr16}. It is known that under mild assumptions, if we have access to enough identically distributed samples, it is possible to accurately estimate the accuracies of the weak labelers, and different minimax optimal methods have been proposed in this setting \citep{zhang:jmlr16,bonald2017minimax}. Our contribution is an adaptive method that allows for this estimation in a non-stationary setting without any prior assumption on the drift.

Our estimation method is based on the technique developed by \cite{bonald2017minimax} that uses the weak labelers' \emph{correlation matrix} to estimate the expertise of each weak labeler in a no-drift setting. In particular, for each $t \geq 1$, we let the correlation matrix $\bm{C}(t) \in [-1,1]^{n \times n}$ be defined as $
    C_{ij}(t) = \Exp_{X \sim D_{t}} \left[  \ell_i(X) \ell_j(X) \right]$ for all $(i,j) \in \{1,\ldots,n\}^2$ 
When there is no distribution drift and under mild assumptions on the bias of the estimates of the weak supervision sources, it is possible to show that a good estimation of the correlation matrix $\bm{C}(t)$ implies a good estimation of the accuracies $\bm{p}(t)$. The assumption on the bias is formalized as follows.
\begin{assumption}
\label{assu:bias}
There exists $\tau > 0$ such that $p_i(t) \geq \frac{1}{2} + \tau$ for all $t \geq 1$ and $i \in \{1,\ldots,n\}$.
\end{assumption}
With this assumption, the following result holds.
\begin{prop}[Lemma 9 of \cite{bonald2017minimax}]
\label{prop:covariance-to-expertise}
Let $\bm{C} \in [-1,1]^{n \times n}$ be a matrix such that $\lVert \bm{C}  - \bm{C}(t)\rVert_{\infty} \leq \epsilon$, and assume $n \geq 3$. Let Assumptions~\ref{assu:sample-independents},~\ref{assu:errors-independents} and~\ref{assu:bias} hold. Then, there exists an estimation procedure that given in input $\bm{C}$, it outputs $\hat{\bm{p}} = (\hat{p}_1,\ldots,\hat{p}_n)$ such that $\lVert \bm{p}(t) - \hat{\bm{p}}\rVert_{\infty} \leq (5/2)\epsilon/\tau^2$.
\end{prop}

The intuition behind the result of \Cref{prop:covariance-to-expertise} is that for a time $t$, for $i \neq j$, the entry $C_{ij}(t)$ is proportional to how much the weak labelers $\ell_i$ and $\ell_j$ agree, and by using \Cref{assu:errors-independents} on the conditional independence of the error of the weak labelers, it can be written as a function of $\bm{p}(t)$ since $
C_{ij}(t) = (2p_i(t)-1)(2p_j(t)-1)
$.
The proposition demonstrates that it is feasible to retrieve $\bm{p}(t)$  by using an estimate of the values $C_{ij}(t)$ for $i \neq j$. Note that both vectors $\bm{p}(t)$ and $\bm{1} - \bm{p}(t)$ would satisfy the constraints given by $C_{ij}(t)$ for $i \neq j$. \Cref{assu:bias} is used to differentiate between those two symmetrical scenarios where two weak labelers are more inclined to agree when they are both likely to be correct or both likely to be incorrect.
The algorithm for the non-drift case presented by \cite{bonald2017minimax}, and the algorithm presented here for the drift case are oblivious to the value of $\tau$.


\section{Algorithm}
As explained in the previous section, our method revolves around the estimation of the correlation matrix $\bm{C}(t)$ at the current time $t$ in order to use \Cref{prop:covariance-to-expertise} and obtain an estimate of $\bm{p}(t)$. 
All proofs are deferred to the supplementary material.
We define $\hat{\bm{C}}^{[r]}(t) \in [-1,1]^{n \times n}$ as $
\hat{\bm{C}}^{[r]}(t) \doteq \frac{1}{r} \sum_{k=t-r+1}^t ( \ell_1(X_k), \ldots, \ell_n(X_k))^T( \ell_1(X_k), \ldots, \ell_n(X_k)) $
The matrix $\hat{\bm{C}}^{[r]}(t) \in [-1,1]^{n \times n}$ is the empirical correlation matrix computed using the latest $r$ samples $X_{t-r+1},\ldots,X_{t}$. This matrix provides the following guarantee on the estimation of $\bm{C}(t)$.
\begin{lemma}
\label{lemma:error-decomposition}
Let $t \geq 1$, let $\delta \in (0,1)$, and let \Cref{assu:errors-independents} hold. The following inequality holds with probability at least $1-\delta$:
\begin{align}
\label{eq:error-decomposition}
\lVert \bm{C}(t) - \hat{\bm{C}}^{[r]}(t) \rVert_{\infty} \leq  \sqrt{\frac{2 \ln(n(n-1)/\delta)}{r}} \nonumber  \\ + 12\sum_{k=t-r+1}^{t-1} \lVert \bm{p}(k) - \bm{p}(k+1)\rVert_{\infty} 
\end{align}
\end{lemma}
\Cref{lemma:error-decomposition} shows that the error of estimating $\bm{C}(t)$ by using the previous $r$ samples can be upper bounded with the sum of two error terms: a \emph{statistical error} and a \emph{drift error}. The statistical error is related to the \emph{sample complexity} of the estimation: it is due to the variance of the estimator  $\hat{C}^{[r]}(t)$, and it decays with rate $O(1/\sqrt{r})$. The drift error is unknown and it quantifies the error introduced due to the distribution shift, and it is measured as the sum of the maximum variation of the accuracy of the weak labelers at each step. The drift error is non-decreasing with respect to $r$. There is a trade-off: we want to choose $r$ to minimize the sum of the statistical error and the drift error.

We remark that adapting to drifting data adds an additional constraint to the sample complexity of the solution. When the samples are drawn from a fixed distribution, a larger sample size reduces the statistical error of the estimates. When the samples are drawn from a drifting distribution, a larger sample size still reduces the statistical error. However, a larger sample also increases the drift error since it uses older samples. Thus, even if the algorithm receives a very long input sequence, the drift poses an upper bound on the sample size that can be used efficiently at any given step.

Our main contribution is an algorithm that \emph{without} any assumption on the drift can provide a close-to-optimal solution of the above trade-off \eqref{eq:error-decomposition}. 
This is a challenging problem, as it is not possible to estimate the drift error, since we only have a single sample from each distribution. Nonetheless, our algorithm 
guarantees an estimation error of the matrix $\bm{C}(t)$ that is essentially up to constant as tight as the value of $r$ that minimizes the right-hand side of \eqref{eq:error-decomposition}. This yields a guarantee on the estimation of  $\bm{p}(t)$ by using \Cref{prop:covariance-to-expertise}.  

The pseudocode of our method is reported in Algorithm~\ref{alg:finalalgorithm}. Our algorithm has the following parameters (a more detailed discussion is deferred to \Cref{subsec:hyper-parameters}):
\begin{itemize}[noitemsep,topsep=0pt,parsep=0pt,partopsep=0pt]
\item $\delta \in (0,1)$ is the failure probability of the algorithm for the estimation at a given time $t$.
\item $m \in \mathbb{N}$ is the maximum number of window sizes evaluated by our algorithm. 
\item The parameter $\beta > 0$ affects the threshold used in our algorithm, and it controls the sensitivity of our algorithm to changes in drift. We empirically show that any value $\beta \in (0,0.1)$ provides a good estimation (\Cref{subsec:hyper-parameters}). 
\end{itemize}

The next theorem gives our algorithm's error guarantee.
\begin{theorem}[Main Result]
\label{thm:main}
Let Assumptions~\ref{assu:sample-independents}, \ref{assu:errors-independents} and \ref{assu:bias} hold. Let $\delta \in (0,1)$, $\beta >0$, and $m \in \mathbb{N}$.   Assume the number of weak labelers is $n \geq 3$. Fix a time $t \geq 1$. There exist an algorithm that outputs $\hat{\bm{p}} = (\hat{p}_1,\ldots,\hat{p}_n)$ such that with probability at least $1-\delta$:
\begin{align*}
   \left\lVert \bm{p}(t) - \hat{\bm{p}} \right\rVert_{\infty} = O\Bigg(  \frac{\beta + \beta^{-1}}{\tau^2}  \min_{1 \leq r \leq \min(t,2^m)} \Bigg(  \sqrt{\frac{\ln(nm/\delta)}{r}} 
    \\ +  \sum_{k=t-r+1}^{t-1} \lVert \bm{p}(k) - \bm{p}(k+1) \rVert_{\infty}\Bigg)  \Bigg) \enspace ,
\end{align*}
where $\tau$ is defined as in \Cref{assu:bias}.
\end{theorem}

A statement of the theorem with the exact constants is provided in \Cref{sec:deferred-proof}. The algorithm evaluates window sizes $\mathcal{R} = \{ r_1 = 2^0, \ldots, r_m = 2^{m-1} \}$. Our method increases the window size ending at time $t$ as long as it does not include significant drift in distribution. As a reference for making this decision we observe that 
 if the samples are identically distributed, then the estimated correlation matrices $\hat{\bm{C}}^{[r_{k+1}]}$ and $\hat{\bm{C}}^{[r_{k}]}$ should be similar up to the statistical error that is proportional to $O(1/\sqrt{r_k})$.
The strategy of our algorithm is based on this intuition. Starting with $k=1$, we iteratively compare the empirical covariance matrix computed respectively with $r_k$ and $r_{k+1}$ samples. If there is minimal drift, the empirical quantity $\lVert \hat{\bm{C}}^{[r_{k+1}]} - \hat{\bm{C}}^{[r_{k}]}_{i,j} \rVert|_{\infty}$ should be comparable to the statistical error due to using $r_k$ samples. If that is the case, we increase the value of $k$. If this empirical quantity is larger, then a significant drift must have occurred. In this case, we can stop and show that using $r_k$ samples is provably close to optimal.
This strategy is implemented in lines 2--7. The \emph{threshold} used as a terminating condition for the iteration of the algorithm is the right-hand side of line 5. The lines 8--12 implement the method that maps a correlation matrix to the accuracies of the weak supervision sources and attains the guarantees of \Cref{prop:covariance-to-expertise} \citep{bonald2017minimax}.

\begin{algorithm}[tb]
\caption{Non-Stationary Accuracy Estimation}\label{alg:finalalgorithm}

\begin{algorithmic}[1]
\STATE \textbf{Input:} $(X_i)_{i=1}^t, ( \ell_i)_{i=1}^n, \mathcal{R} =\{ r_1, \ldots, r_m \}$, $\beta$, $\delta$.
\vspace{2pt}
\STATE $A_{n,m,\delta} \gets \sqrt{ 2 \ln[(2m-1)\cdot n(n-1)/\delta]} $
\STATE $k \gets 1$
\WHILE{$(k \leq m-1)$ and $(r_{k+1} \leq t)$}
\STATE \textbf{if } $\lVert \hat{\bm{C}}^{[r_{k+1}]}(t)  - \hat{\bm{C}}^{[r_k]}(t) \rVert_{\infty} \leq   A_{n,m,\delta} \left[ \frac{2\beta}{\sqrt{r_k}} + \sqrt{\frac{1-\frac{r_k}{r_{k+1}}}{r_k}}  \right]$
    \textbf{ then } $k \gets k+1$
\STATE \textbf{else break}
\ENDWHILE

\STATE $\hat{\bm{C}} \leftarrow \hat{\bm{C}}^{[r_k]}(t)$
\FORALL{$h \in \{1, \ldots, n\}$}
\STATE $(i,j) \leftarrow \mathrm{argmax}_{ i\neq j, j \neq h, i \neq h} \left| \hat{C}_{ij}\right|$
\STATE \textbf{if} $\hat{C}_{ij} = 0$ \textbf{ then } $\hat{p}_h \leftarrow 1/2$
\STATE \textbf{else }
 $ \hat{p}_h \leftarrow \left(1 + \sqrt{\left| \frac{\hat{C}_{ih} \hat{C}_{hj}}{\hat{C}_{ij}}\right|}\right)/2$
\ENDFOR
\STATE \textbf{return} $\hat{p} = (\hat{p}_1,\ldots,\hat{p}_n)$
\end{algorithmic}
\end{algorithm}

\paragraph{Computation and Memory.} We optimize the memory and computational efficiency of the algorithm in an online setting
by storing only the most recent $r_m = 2^{m-1}$ data points at each step where $m$ is the maximum number of window sizes considered by our algorithm. 
The value of $m$ provides a trade-off between the quality of estimation and the use of memory and computational resources. A larger maximum window size can yield a better estimation when the data is stationary. However, it is sufficient to use a small value of $m$: for any $\epsilon >0$, if we set $m=O(\log(1/\epsilon))$, then we incur an additional additive estimation error $\epsilon$ compared to an algorithm that has access to all the past samples.
Also, the empirical covariance matrices can be maintained efficiently: for any window size $r \in \mathcal{R}$, the 
matrix $\hat{C}^{[r]}(t)$ differs from $\hat{C}^{[r]}(t-1)$ by only two samples, thus it can be updated in time $O(n^2)$.

\textbf{Failure Probability.} The theorem guarantees that the estimation is correct (i.e., it satisfies the upper bound of the statement) at any fixed time step $t$ with probability at least $1-\delta$, where $\delta > 0$ can be set arbitrarily small.
 Thus, over a horizon $T \geq 0$, the algorithm is expected to fail in at most $T\delta$ steps. It is important to note that the algorithm 
 is guaranteed to recover from failures. In particular, since the maximum window size is $r_m$,  the event of giving a wrong estimation at time $t+r_m$ is independent of the outcomes of all events before time $t+1$. 

\section{Empirical Evaluation}

 \begin{figure*}[ht]
\includegraphics[width=.49\textwidth]{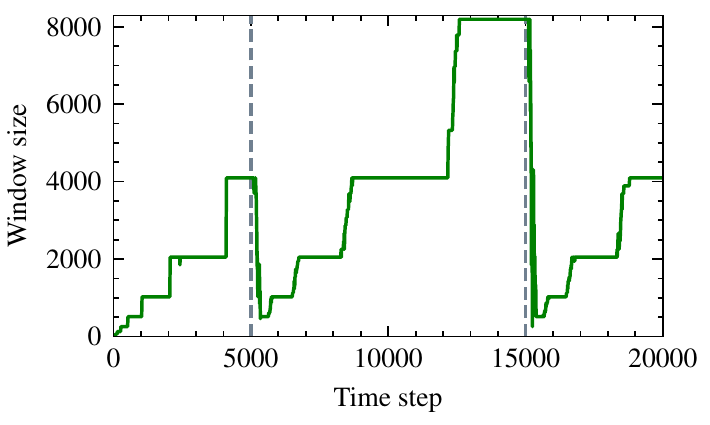}
\hfill
\includegraphics[width=.49\textwidth]{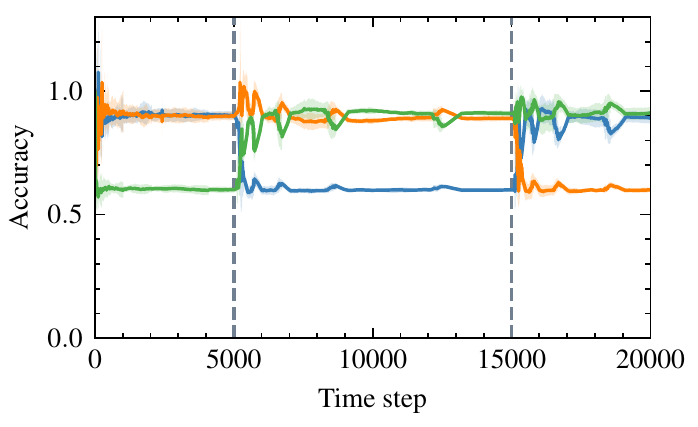}
\caption{We report the window size chosen by the algorithm (left) and the estimated accuracies of each weak labeler over time (right). 
The vertical lines represent when a change in distribution occurs.} 
\label{fig:synthetic}
\end{figure*}


We demonstrate the functionality and the advantage of our algorithm over fixed-window-size strategies in several experimental settings, including synthetic data, image recognition, and video classification. Due to space constraints, we present only part of the experimental results in this section. Additional experiments and details are reported in the supplementary material.
\paragraph{Setup.}
At each time step, we receive an unlabeled example which must be labeled based on the available weak labelers.
We use Algorithm~\ref{alg:finalalgorithm} to estimate the accuracies of the weak labelers.
We then make a prediction for the current time step's example by weighting the vote of each labeler proportionally to its estimated accuracy using the weighting $\bm{w}^*$ described in Equation~\eqref{eq:w^*}.
For all experiments, we run our algorithm with  $m=20$, $\mathcal{R}=\{2^0,2^1,\dots, 2^{19}\}$, and $\beta = \delta = 0.1$ (see Appendix~\ref{subsec:hyper-parameters}).
We compare with the following baselines algorithms:  (1) a majority vote of the weak labelers at this step; and (2) non-adaptive fixed-window-size algorithms, one for each size in $\mathcal{R}$.
The latter algorithms are the same as Algorithm~\ref{alg:finalalgorithm}, except that we use a fixed size window to estimate $\hat{\bm{C}}$ (line 7).
Since the triplet method for estimating accuracies (lines 8-12) is not constrained to return a probability between 0 and 1, we clip the estimated accuracies to the interval $[0.1,0.9]$.
The fixed-window-size algorithms are analogous to previous work on crowdsourcing with drift that also determines a priori how much information to use from the past, depending on an assumption that constraints the magnitude of the drift. In practice, it is not possible to run those algorithms from the previous work, since we cannot estimate the drift as we have access to only a single sample from each distribution.
The code for the experiments is available online \footnote{\url{https://github.com/BatsResearch/mazzetto-aistats2025-code}}.

\subsection{Synthetic Data}
\label{subsection:synthetic}

We first show how our algorithm adapts to changing input distributions with an artificial experiment on synthetic data that satisfies all of our assumptions. The algorithm receives input from three weak labelers $(n=3)$, and the input stream has $4 \cdot T$ data points with $T=5000$. The data is partitioned into three contiguous blocks of size $T$, $2T$, and $T$. The accuracies of the weak labelers do not change within the same block, but do change between blocks. In particular, for each block, two weak labelers have a high accuracy equal to $0.9$, and the other one has a low accuracy equal to $0.6$. The weak labeler with low accuracy is different in each block. We remark that our algorithm is oblivious to this partitioning of the data.

In Figure~\ref{fig:synthetic}, we plot the window size used by the adaptive algorithm and its estimates of the accuracies of each weak labeler in each time $t$, $1 \leq t \leq 4T$. The reported results are an average over $10$ independent generations of this synthetic data. The main observation is that our algorithm correctly identifies a change in distribution, and reduces the window size whenever it transitions to the next block. This allows for a very good estimation of the weak labeler accuracies, as the algorithm uses data mostly from the current block.

\begin{figure}[ht]
\includegraphics[width=\columnwidth]{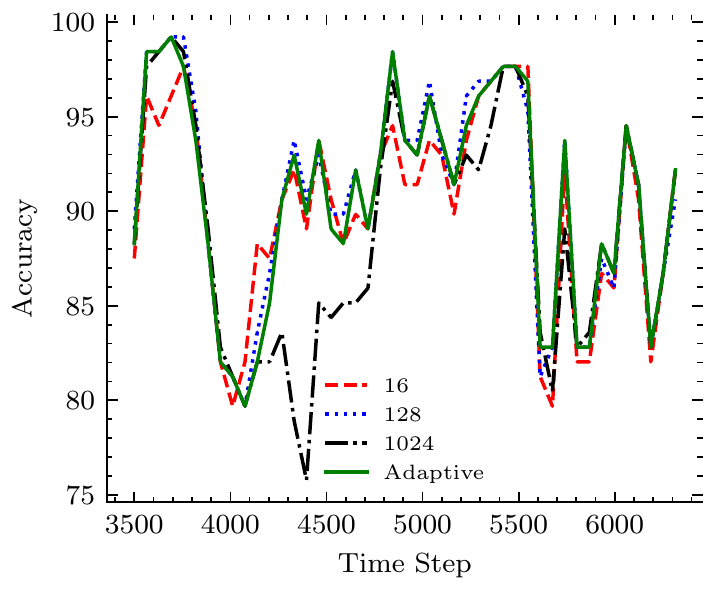}
\caption{ 
We report the accuracy of different fixed-window-size strategies and our adaptive algorithm for the Tennis Rally dataset over time (single run, \emph{permute}). For each time step $t$, the reported accuracy is an average over the next $128$ time steps. 
The plot shows that no fixed-window-size strategy is consistently good, while the adaptive strategy consistently matches the best strategy at any given time.}
\label{fig:accuracy_over_time}
\end{figure}
\begin{figure*}[ht]
\includegraphics[width=.49\textwidth]{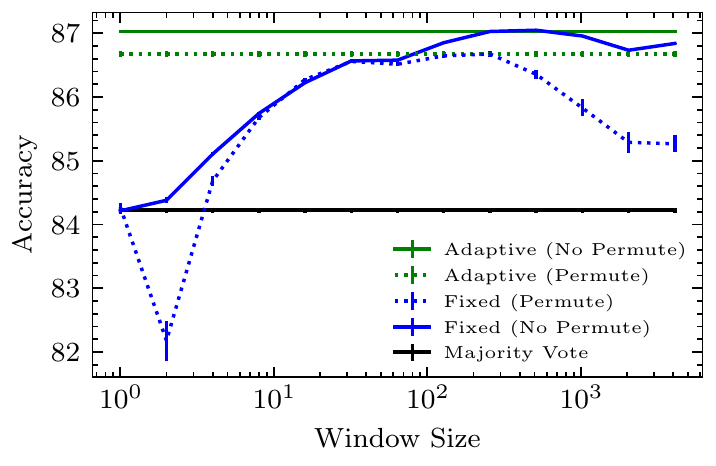}
\hfill
\includegraphics[width=.49\textwidth]{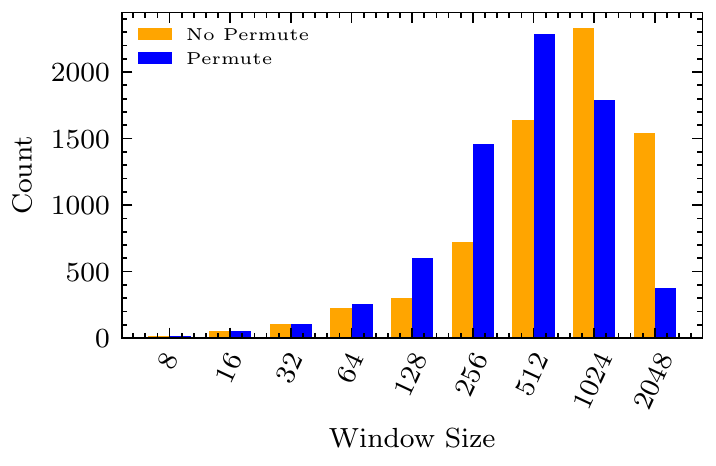}
\caption{In the left plot, we report the average accuracy of the dynamically selected window sizes (our algorithm) and different fixed-window-size strategies for the Tennis Rally dataset. In the right plot, we report the histogram of the window sizes chosen by our algorithm. The reported results are for both experimental setups \emph{permute} and \emph{no permute}. The average and standard deviations are over $30$ random runs, where the randomness of each run is due to the abstensions, and the shuffling of the weak labelers for the \emph{permute} setup. }
\label{figure:accuracy-window-size}
\end{figure*}
Clearly, there is a delay before our algorithm can correctly identify the distribution change since it needs to collect enough data from the new block to assess that a significant change happened. As a result, the estimation of the weak labelers' accuracy is worse for the data right after a block change. The variation in the accuracy estimates in the middle of a block is due to the window size selection strategy of the algorithm: whenever the algorithm increases the window size, the larger window includes in the following few steps a small number of samples from the previous block, resulting in a small additional error in the estimation.
Previous approaches, which predetermine the amount of past information to use based on a parameter $\Delta$ that measures the magnitude of the drift, cannot address the drift scenario presented in this synthetic experiment. In this case, the duration of each stationary period varies, altering the number of samples we wish to extract from each period.

\subsection{Video Analysis}
\label{sub:video}
We evaluate our algorithm on three video analysis tasks: \textbf{Basketball} \citep{sala:neurips19}, \textbf{Commercial}, and \textbf{Tennis Rally} \citep{fu2020fast}. For each of those datasets, we are given a sequence of frames of a video, and a set of weak supervision sources \citep{zhang:neurips21}.
We report here the results for the Tennis Rally dataset, where the goal is to identify tennis rallies during a broadcast match.  Additional details for this dataset, and the results for the other video analysis tasks are presented in the supplementary material.

The Tennis Rally dataset has $6$ weak supervision sources that can abstain on some of the input frames. Since our method requires the weak labelers to provide an output at each step, we map each abstention to a random label $\pm 1$. For each of the video analysis tasks, we use two experimental setups. In the first experiment (\emph{no permute}), we simply use the original dataset. In the second experiment  (\emph{permute}), we introduce an additional source of drift for the weak labelers. In particular, at each time step, with probability $10^{-3}$ we randomly shuffle the names of the weak labelers.

The first observation is that the performance of a given fixed-window-size strategy can change over time, as displayed in Figure~\ref{fig:accuracy_over_time}. In particular, the best choice of window size also changes over time.  Our algorithm obtains an overall good performance by consistently adapting its window size based on the input sequence. In the left plot of Figure~\ref{figure:accuracy-window-size}, we compare the average accuracy of our algorithm with the majority vote and the fixed-window-size strategies for both experimental setups (\emph{permute} and \emph{no permute}). 
Our algorithm achieves an accuracy that is competitive with respect to the optimal fixed-window-size strategy. 
We remark that the accuracy values in these figures are computed using labeled data that is \emph{not} available to the algorithm. Thus, an algorithm cannot evaluate the accuracy of the fixed-window strategies to choose the best one among them. Additionally, no fixed-window-size strategy can perform optimally in both drift settings as its accuracy is inherently linked to the data's drift patterns. On the other hand, our algorithm can automatically adjust to the drift and do at least as well as the best fixed-window-size strategy.


To further illustrate the adaptivity of our algorithm, we report in the right plot of Figure~\ref{figure:accuracy-window-size} a histogram of the window sizes chosen by our algorithm during its execution. This plot demonstrates that our algorithm varies its chosen window size throughout its execution. In particular, our algorithm indeed selects smaller window sizes in the \emph{permute} setup, since it adapts to the additional drift due to the random shuffling of the identities of the weak labelers. We can also observe that the accuracies of the fixed-window-size strategies follow the phenomenon outlined in our theory: small window sizes and larger window sizes are less competitive than a proper selection of the window size depending on the drift. Using a large fixed-window-size strategy still provides a good solution for the \emph{no permute} setup. This is most likely because there are a few weak labelers that are consistently accurate throughout the whole input sequence. On the other hand, for the \emph{permute} setup, large fixed-window-size strategies exhibit a lower accuracy as weak labelers can have different performances throughout the window due to the random shuffling.  Nevertheless, our algorithm can still adaptively recognize the identities of the accurate weak supervision sources.
We also observe that the fixed-window strategy with size $1$ is equivalent to the majority vote. This is most likely the reason why this strategy can be more competitive than other small fixed-window-size strategies that learn noisy weights based on little data.

\vspace{-5pt}
\section{Conclusion}
This paper presents the first method with rigorous guarantees for learning from multiple noisy sources of labels in non-stationary settings \emph{without} any prior assumptions about the nature of the changes over time. 

\textbf{Limitations and Future Work.}
The method presented in this paper does not extend to multi-class classification. One can follow the heuristic proposed by \cite{fu2020fast}, and execute multiple one-versus-all classifiers. However, this heuristic does not provide any formal analysis on the obtained result, and the outcome of different one-versus-all classifiers may not be consistent. Thus, a provable multi-class classification under drift is an interesting open problem.

\subsection*{Acknowledgements}
We sincerely thank the anonymous reviewers for their valuable feedback and constructive suggestions, which have helped improve the quality of this work. Alessio Mazzetto was supported by a Kanellakis Fellowship at Brown University. We gratefully acknowledge support from Cisco. This material is based on research sponsored by Defense Advanced Research Projects Agency (DARPA) and Air Force Research Laboratory (AFRL) under agreement number FA8750-19-2-1006 and by the National Science Foundation (NSF) under award IIS-1813444. The U.S. Government is authorized to reproduce and distribute reprints for Governmental purposes notwithstanding any copyright notation thereon. The views and conclusions contained herein are those of the authors and should not be interpreted as necessarily representing the official policies or endorsements, either expressed or implied, of Defense Advanced Research Projects Agency (DARPA) and Air Force Research Laboratory (AFRL) or the U.S. Government. Disclosure: Stephen Bach is an advisor to Snorkel AI, a company that provides software and services for data-centric artificial intelligence.

\bibliography{bibliography}


\newpage
\onecolumn
\newpage
\appendix
\section{Additional Experiments}

We provide further details about our experimental setup and also show additional results on other tasks. All experiments were conducted on a M2 Max Macbook Pro with 32GB RAM.
This section is organized as follows:
\begin{enumerate}
\item In \Cref{app:video-classification}, we report additional details for our experiments on the video classification tasks (\Cref{sub:video}). In particular, we report the experimental results for the Basketball and Commercial datasets, and also provide additional details on the Tennis dataset.
\item In \Cref{subsec:awa2_exp}, we provide experimental results for an image classification task with artificial drift that uses the Animals with Attributes 2 (AwA2) dataset.
\item In \Cref{subsec:hyper-parameters}, we discuss the hyper-parameters of our algorithm.
\end{enumerate}

\subsection{Video Classification}
\label{app:video-classification}
In this subsection, we provide additional details and results for the experiments on the video classification tasks (\Cref{sub:video}).

In this set of experiments, we use three datasets: \textbf{Commercial}, \textbf{Tennis Rally}, and \textbf{Basketball}.

\textbf{Commercial} \citep{fu2019rekall}: the goal is to identify commercial segments from a TV news broadcast.

\textbf{Basketball} \citep{fu2020fast}: the goal is to identify basketball videos in a subset of ActivityNet \citep{caba2015activitynet}.

\textbf{Tennis Rally} \citep{fu2020fast}: the goal is to identify tennis rallies during a tennis match.

We use the version of those datasets provided by the weak supervision benchmark platform Wrench \citep{zhang:neurips21}.  For each of those datasets, we only use the training data, since it contains the largest number of data points. We use the weak supervision sources provided by Wrench for each of those tasks. Table~\ref{table:video_overview} provides additional information on the number of data points and weak supervision sources for each task.


\begin{table}[h!]
\centering
\caption{Number of weak labelers and data points for the three datasets.}
\begin{tabular}{lcc}
\toprule
Dataset    & Number of Weak Labelers & Number of Data Points \\  
\midrule
Basketball & 4                       & 6959                  \\
Commercial & 4                       & 64130                 \\
Tennis     & 6                       & 17970                 \\    
\bottomrule
\vspace{1em}
\end{tabular}
\label{table:video_overview}
\end{table}

The sequence of data points described the sequence of frames of the video. The weak supervision sources provide a vote for each frame, and the goal is to combine their votes to provide the correct classification of the frame. In the original dataset, the weak supervision sources are allowed to abstain on some of the data points. To fit those datasets to our theoretical framework, we map each abstension to a random vote in $ \{-1,1\}$.

As described in \Cref{sub:video}, in the ``no permute'' case we use the dataset as is. In the ``permute'' case, we randomly shuffle the names of the weak labelers with probability $10^{-3}$ at each step.

For each of those datasets, we run an experimental evaluation across $30$ runs. In each run, there is randomness introduced in the mapping of the abstentions to random vote, and also in the random shuffling of the name of the weak labelers for the ``permute'' case.

\subsubsection{Additional Results}
In Figure~\ref{figure:accuracy-window-size-commercial} and Figure~\ref{figure:accuracy-window-size-basketball}, we report experimental results for respectively Commercial and Basketball. Those results corroborate the findings obtained with the Tennis Rally dataset (\Cref{sub:video}). 
In particular, we show again that smaller window sizes have higher accuracy in the ``permute'' case due to the additional drift.  Accordingly, the right plot of Figure~\ref{figure:accuracy-window-size-commercial} clearly shows that the adaptive algorithm favors smaller window size in the ``permute'' case for the Commercial dataset.

We highlight that although the Basketball dataset contains a lot of noise due to having the highest percentage of abstension and the smallest number of data points, the shape of the plot of the left plot of \Cref{figure:accuracy-window-size-basketball} still resembles what we would expect from our theoretical framework: smaller and larger window sizes are less competitive (accuracy-wise) than a proper selection of the window size depending on the drift.


\begin{figure*}[ht]
\includegraphics[width=.49\textwidth]{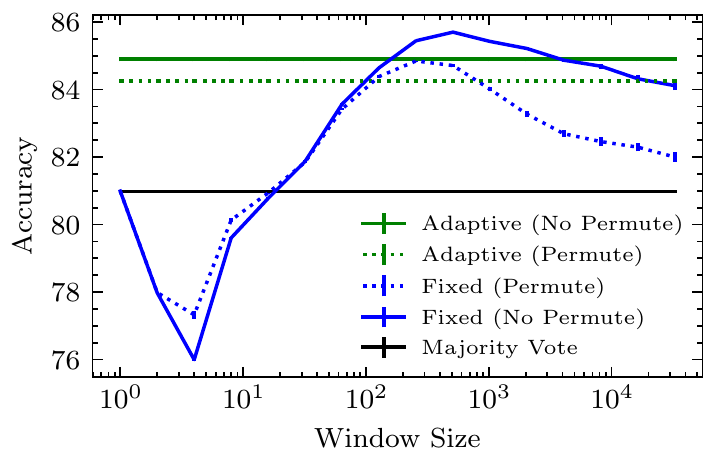}
\hfill
\includegraphics[width=.49\textwidth]{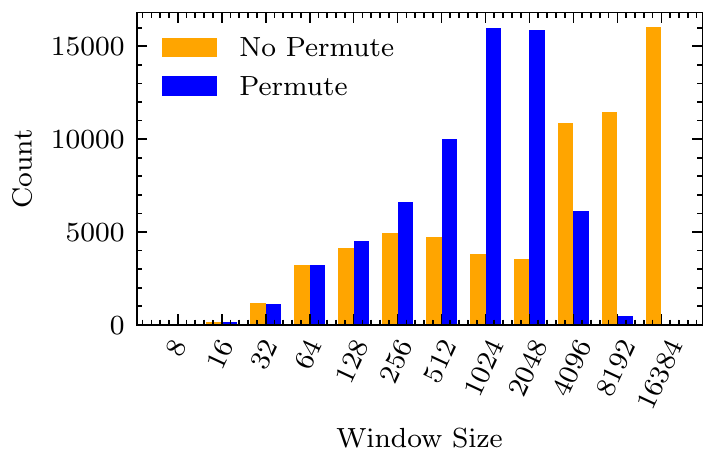}
\caption{(\textbf{Commercial}). In the left plot, we report the average accuracy of the dynamically selected window sizes (our algorithm) and different fixed-window-size strategies for the Commercial Dataset. In the right plot, we report the histogram of the window sizes chosen by our algorithm.}
\label{figure:accuracy-window-size-commercial}
\end{figure*}
\vfill
\begin{figure*}[h!]
\includegraphics[width=.49\textwidth]{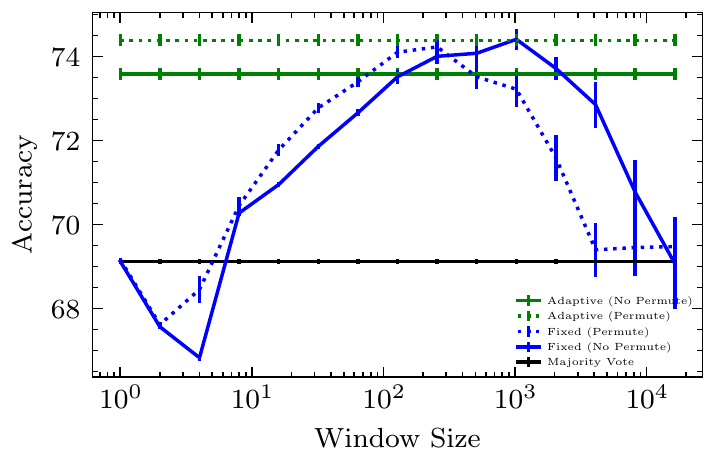}
\hfill
\includegraphics[width=.49\textwidth]{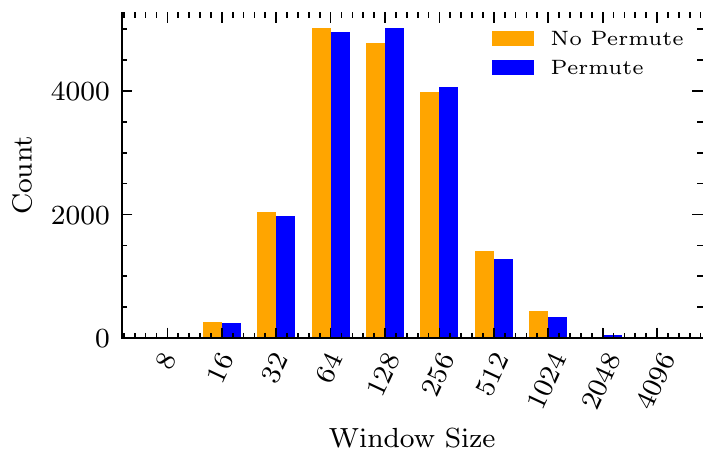}
\caption{(\textbf{Basketball)}. We report the same results of \Cref{figure:accuracy-window-size-commercial} but for the Basketball dataset.}
\label{figure:accuracy-window-size-basketball}
\end{figure*}

For completeness, in Figures \ref{figure:tennis_F1_vs_window_size} and \ref{figure:commercial+basketball_F1_vs_window_size}, we also report the F1 Score (the harmonic mean of precision and recall) obtained by our adaptive algorithm and the fixed-window-size strategies across all three tasks. We observe that our observations regarding the accuracy numbers also applies to the F1 score.

\begin{figure*}[h!]
\centering
\includegraphics[width=.49\textwidth]{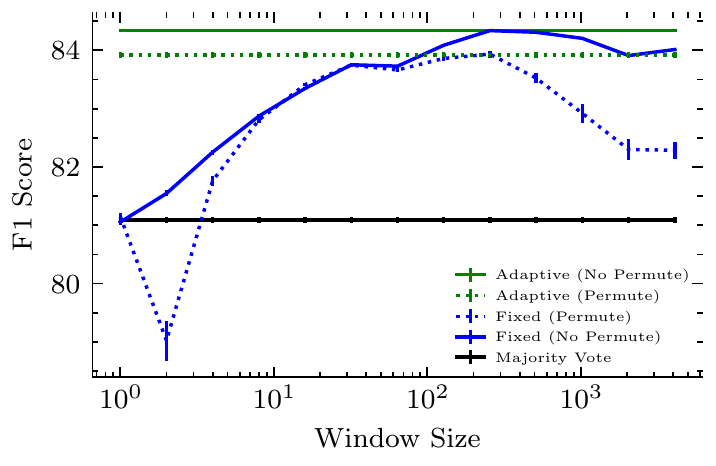}
\caption{F1 Score on the Tennis dataset.}
\label{figure:tennis_F1_vs_window_size}
\end{figure*}

\begin{figure*}[h!]
\includegraphics[width=.49\textwidth]{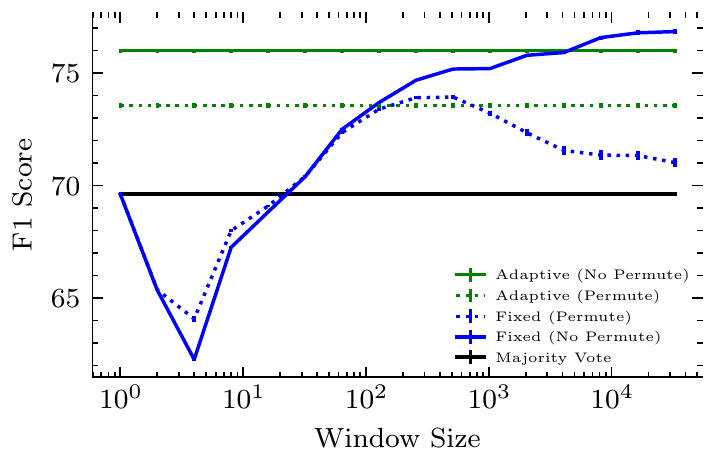}
\hfill
\includegraphics[width=.49\textwidth]{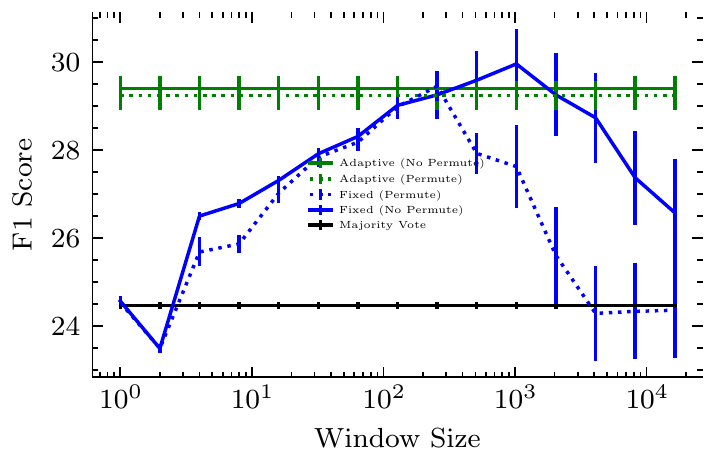}
\caption{F1 Score on the Commercial dataset (left) and on the Basketball dataset (right)}
\label{figure:commercial+basketball_F1_vs_window_size}
\end{figure*}

We also plot the weights computed by our algorithm for the weak supervision sources over a single run for both the settings ``permute'' and ``no permute'' (Figure~\ref{figure:tennis_learner_weights} and~\ref{figure:commercial_learner_weights}). We remind that the weights of the algorithm are computed according to \eqref{eq:w^*} given the adaptive algorithm's estimated accuracies.

In the no permute case, there are no sudden changes in the weights of the learners. This also suggests that the original dataset does not exhibit a significant drift on which weak labelers are accurate. However, in the ``permute'' case there are sharp changes in the weights of the learner, which correspond to when a shuffle of the names of the weak labelers occurs in the dataset. This shows that our algorithm can adapt to the drift introduced over the weak labelers.


\begin{figure*}[h!tbp]
\includegraphics[width=.49\textwidth]{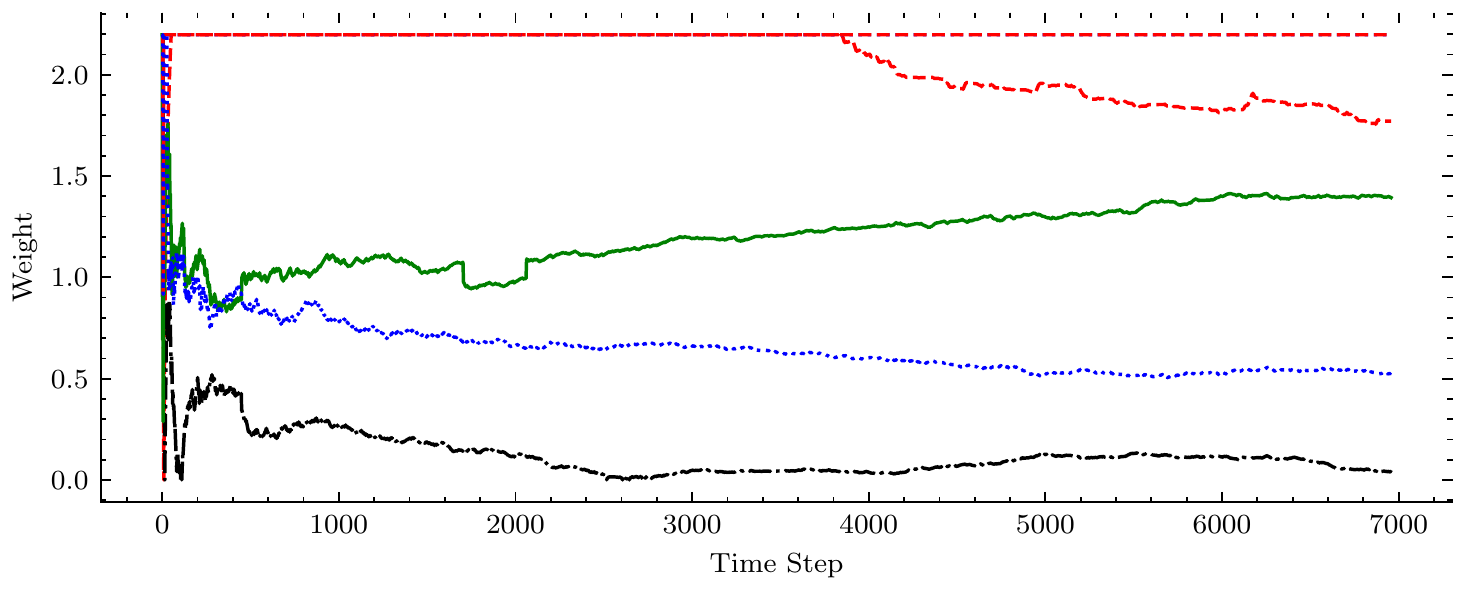}
\hfill
\includegraphics[width=.49\textwidth]{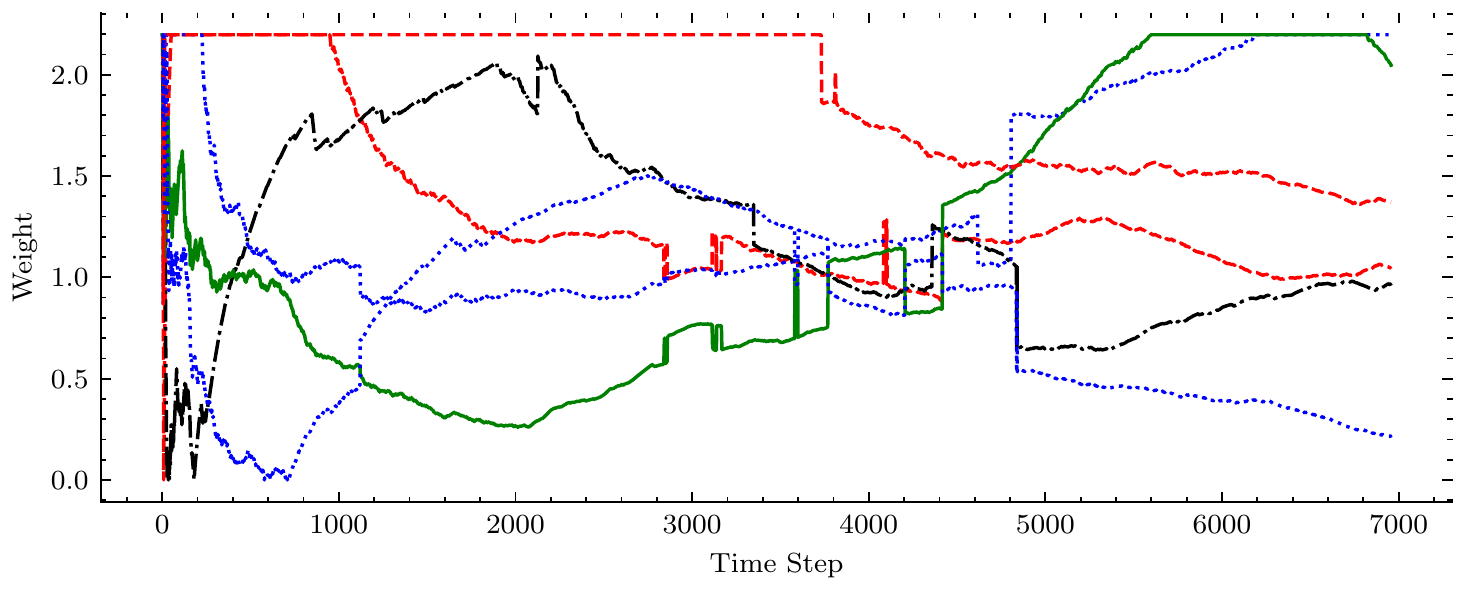}
\caption{\textbf{Tennis Rally.} The weights given to each weak labeler over time by a single run of our adaptive algorithm. The left is the ''no permute'' case and the right is the ''permute'' case.}
\label{figure:tennis_learner_weights}
\end{figure*}

\begin{figure*}[h!tbp]
\includegraphics[width=.49\textwidth]{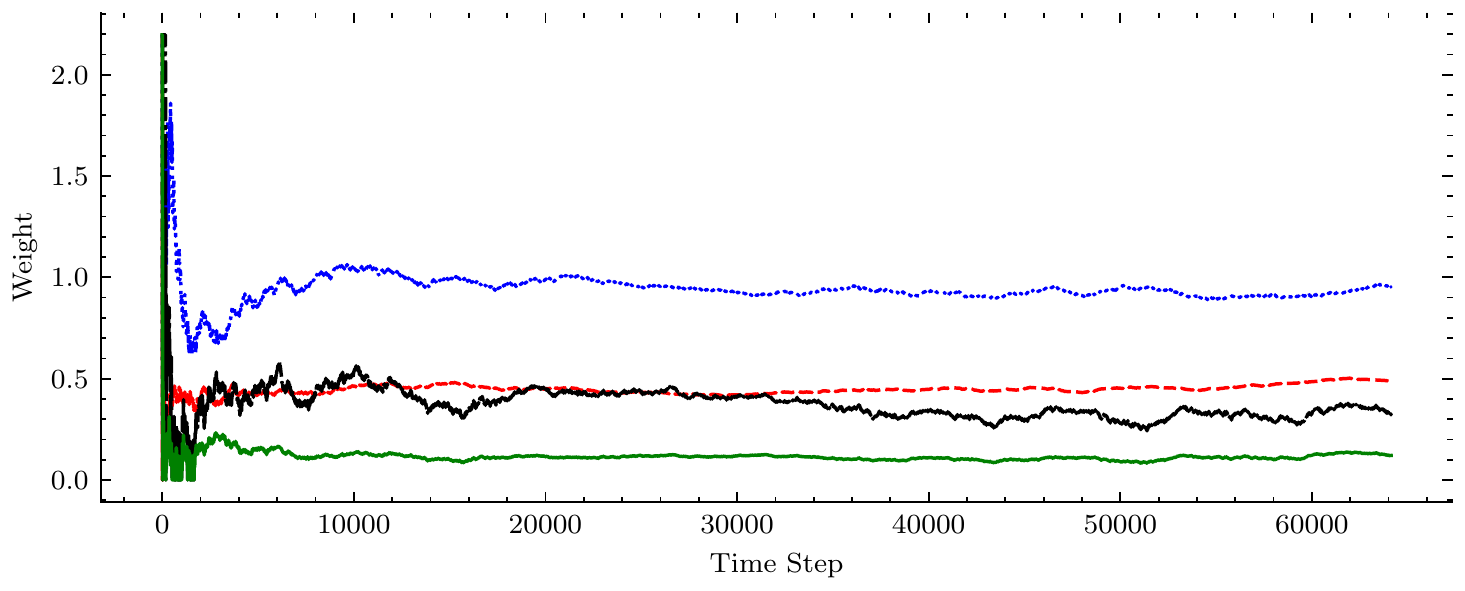}
\hfill
\includegraphics[width=.49\textwidth]{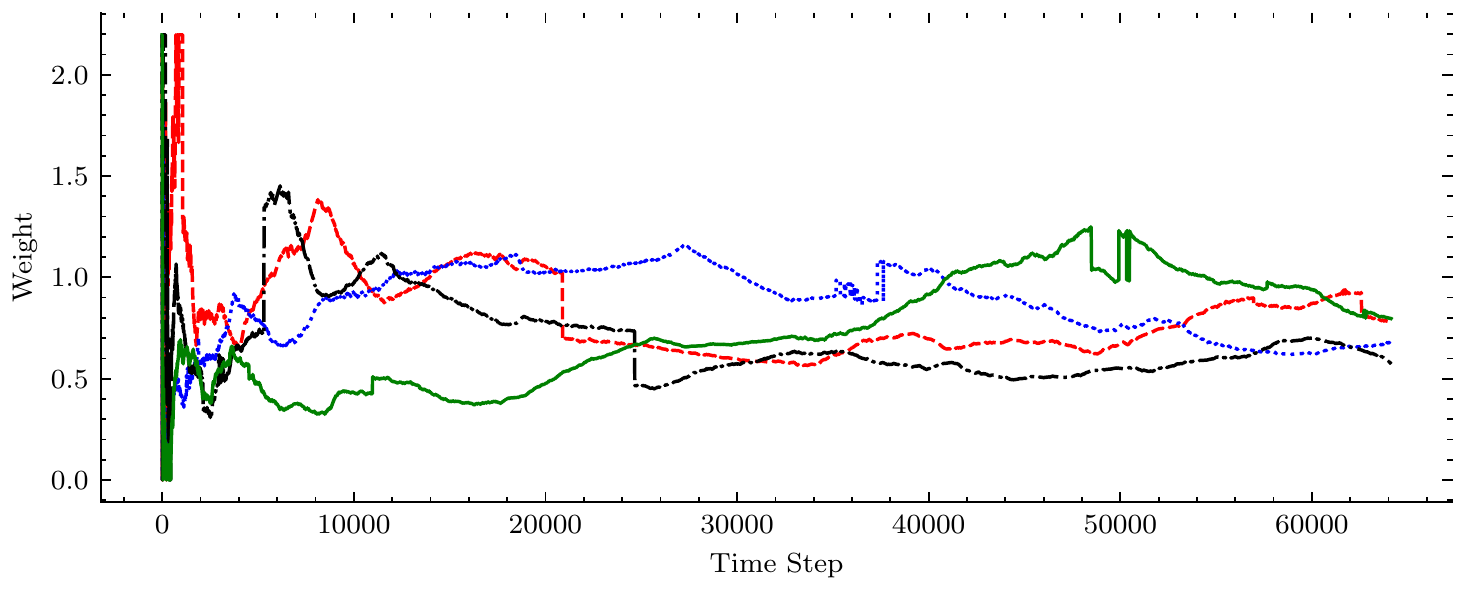}
\caption{\textbf{Commercial.} The weights given to each weak labeler over time by a single run of our adaptive algorithm. The left is the ''no permute'' case and the right is the ''permute'' case.}
\label{figure:commercial_learner_weights}
\end{figure*}

\newpage

\subsection{Image Classification}
\label{subsec:awa2_exp}

In this subsection, we report the experimental results for an image classification task with artificial drift built from the \emph{Animals with Attributes2} (AwA2) dataset \citep{xian2018zero}.

The {Animals with Attributes2}  dataset consists of images of animals from $50$ disjoint classes, that are split into $40$ seen classes, used for training, and $10$ unseen classes, used for testing.
The dataset also provides the relations among $85$ attributes (e.g., ``patches'') and classes through a binary class-attribute matrix, where each entry indicates whether animals from a certain class exhibit an attribute or not.
Following previous work~\citep{mazzetto:aistats21,mazzetto:icml21}, we obtain weak supervision sources by fine-tuning ResNet-18 models~\citep{he2016deep} on the seen classes to detect each of the attributes.

We use this dataset to construct a binary classification task with artificial drift. We define two target classes over the unseen test classes. The first target class contains images from the classes ``horse'' and ``sheep''; the second target class contains images from classes ``giraffe'' and ``bobcat.''
We use the class-attribute matrix to identify attributes that are helpful to distinguish between those two target classes.
An attribute is helpful if 1) it appears only in one of the target classes and 2) it consistently appears or does not appear in both classes of each target class.
Using this criteria, we choose the attribute detectors for ``black'', ``white'', ``orange'', ``yellow'', ``spots'', and ``domestic'' attributes as weak supervision sources.

To create a dataset with drift, we sample $5T$ images with repetition from the selected classes with $T=4000$.
We partition the data into five contiguous blocks of size $T$.
In the first block, we sample from ``sheep'' and ``bobcat'' classes with a probability of $0.1$ and from ``horse'' and ``giraffe'' classes with a probability of $0.9$. To create drift, we alternate the probability of sampling from each of the subclasses between $0.1$ and $0.9$ for consecutive blocks.



In Figure~\ref{fig:awa2_acc}, we visualize the accuracy of the adaptively selected window sizes and multiple fixed window sizes over time.
As expected, the accuracy of fixed window sizes changes over time.
For example, small window sizes achieve better accuracy shortly after a distribution shift occurs by limiting the number of out-of-distribution samples,
and large window sizes achieve better accuracy toward the end of each block by using more samples from the same distribution.
On the other hand, our algorithm successfully detects the drift and selects the best window size for each time step accordingly.
As a result, our algorithm maintains a close-to-optimal performance for most of the time steps.
These results emphasize that the optimal window size itself can change over time.

We report the window sizes selected by our algorithm at each time step in Figure~\ref{fig:awa2_ws_per_ts}.
Consistent with previous results on synthetic data, our algorithm successfully detects the drift and selects small window sizes to limit out-of-distribution samples.
At the same time, for stationary periods, our algorithm selects large window sizes to include more samples from the same distribution.

\begin{figure*}
  \hspace*{-0.5in}
  \centering
  \includegraphics[width=0.9\textwidth]{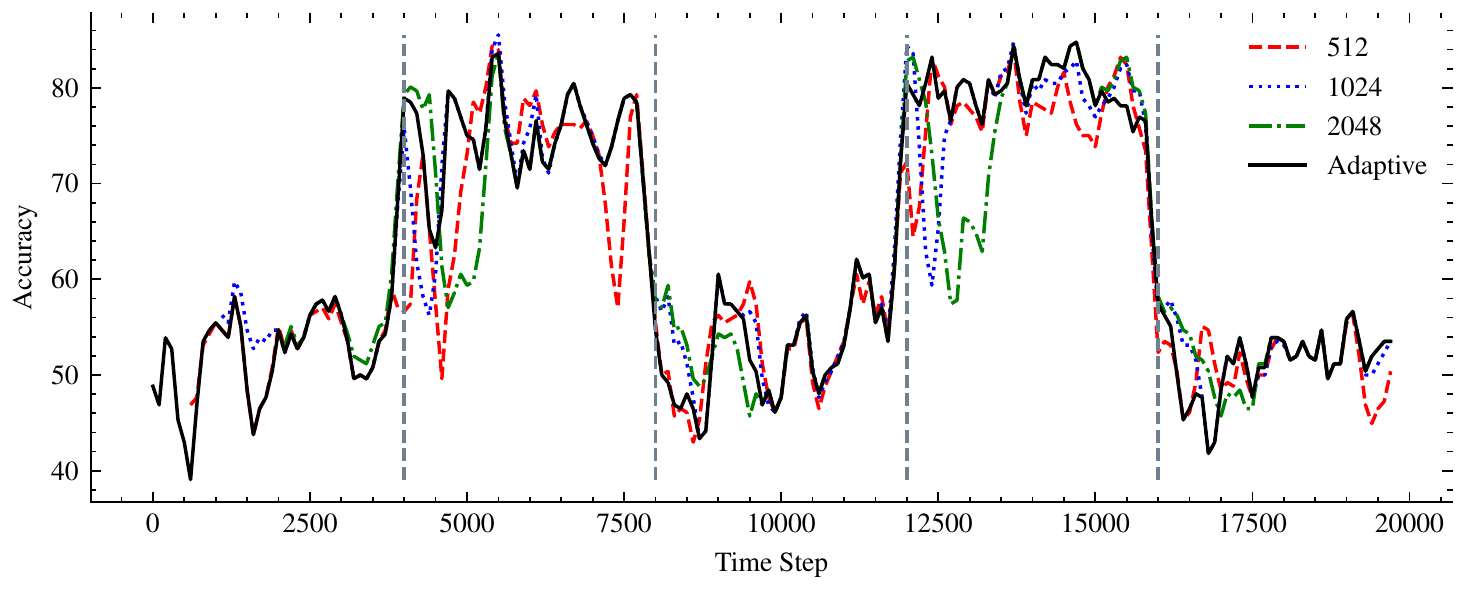}
  \caption{Accuracy of different window sizes for each time step for AwA2 dataset. The vertical lines mark a distribution shift. The accuracies are obtained from one random execution. Accuracies for window size $r$ are only reported for time steps $t \geq r$. For each time step $t$, the reported accuracy is an average over the next $256$ time steps, $X_t, \ldots, X_{t + 256 - 1}$.}
\label{fig:awa2_acc}
\end{figure*}

\begin{figure}
\centering
\includegraphics[width=0.5\textwidth]{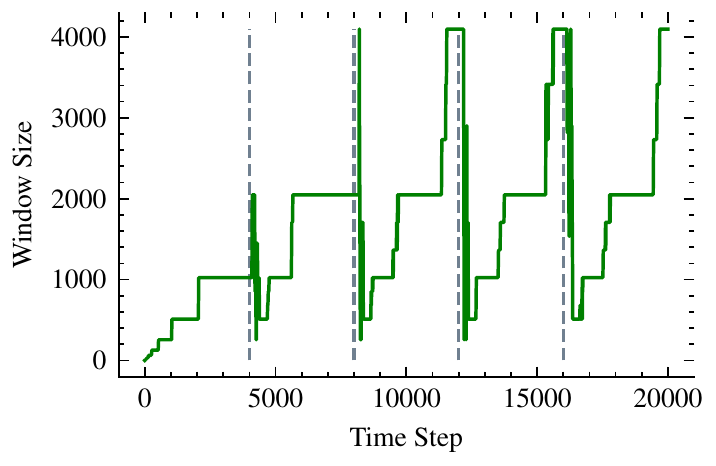} 
\caption{Adaptively selected window sizes for each time step for AwA2 dataset. Vertical lines mark a distribution shift. Each point is an average of three random executions.}
\label{fig:awa2_ws_per_ts}
\end{figure}

In Figure~\ref{fig:awa2_acc_vs_ws} and Table~\ref{table:awa2_acc_all}, we plot the average accuracy of the dynamically selected window sizes and multiple fixed window sizes for the AwA2 dataset.
Using a proper window size is crucial for achieving good performance.
Using too large or too small window sizes decreases the accuracy by up to $7$ percentage points.
On the other hand, our algorithm adapts to each time step and selects a close-to-optimal window size without any prior knowledge about the drift.
As a result, adaptive window sizes achieve better or comparable accuracy to any fixed strategy. 
Although some fixed strategies have a competitive accuracy on average, their accuracy changes significantly for different time steps.

\begin{figure}[h!]
\centering
\centering
\includegraphics[width=0.6\textwidth]{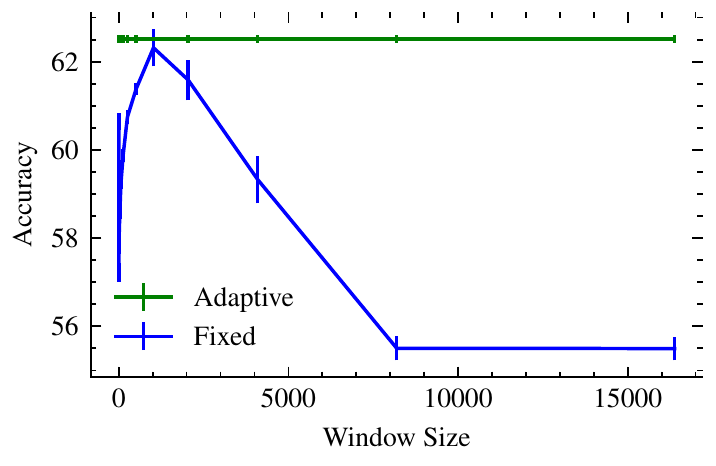} 
\caption{Accuracy of the dynamically selected window sizes as well as different fixed-window-size strategies for the AwA2 dataset. For each window size, the average accuracy and standard error of the mean are reported over three random executions.}
\label{fig:awa2_acc_vs_ws}
\end{figure}

\begin{table}[h!]
\centering
\caption{Accuracy of the dynamically selected window sizes as well as different fixed-window-size strategies for the AwA2 dataset. For each window size, the average accuracy and standard error of the mean are reported over three random executions.}
\begin{tabular}{lc} 
\toprule

Window Size & Accuracy \\
\midrule
Adaptive   & $62.51 \pm 0.09$ \\
All Past        & $55.49 \pm 0.27$ \\
Majority Vote          & $60.64 \pm 0.20$ \\
$2$          & $57.22 \pm 0.21$ \\
$4$          & $58.72 \pm 0.19$ \\
$8$          & $57.75 \pm 0.12$ \\
$16$         & $58.16 \pm 0.07$ \\
$32$         & $58.56 \pm 0.13$ \\
$64$         & $59.23 \pm 0.10$ \\
$128$        & $59.87 \pm 0.15$ \\
$256$        & $60.75 \pm 0.17$ \\
$512$        & $61.38 \pm 0.13$ \\
$1024$       & $62.32 \pm 0.42$ \\
$2048$       & $61.59 \pm 0.46$ \\
$4096$       & $59.33 \pm 0.54$ \\
$8192$       & $55.50 \pm 0.27$ \\
$16384$      & $55.49 \pm 0.27$ \\

\bottomrule
\vspace{1em}
\end{tabular}
\label{table:awa2_acc_all}
\end{table}

\newpage
\subsection{Hyper-parameters}
\label{subsec:hyper-parameters}
In this section, we discuss the hyper-parameters of the adaptive algorithm. The discussion of this section refers to the more general statement of our main result (\Cref{thm:main-2}). Our algorithm has the following hyper-parameters:
\begin{itemize}[noitemsep,topsep=0pt,parsep=0pt,partopsep=0pt]
\item The value $\delta \in (0,1)$ represents the failure probability of the algorithm. 
\item The sequence $\mathcal{R} = \{ r_1, \ldots, r_k \}$ represents the possible window sizes that the algorithm considers. In order to obtain better guarantees in \Cref{thm:main-2}, we look for a sequence $\mathcal{R}$ such that: $(i)$ the minimum ratio between consecutive elements $\gamma_m$ is large, as this avoids comparing window sizes that are very similar with one another and for which it is very hard to detect if drift occurred; $(ii)$  the maximum ratio between consecutive elements $\gamma_M$ is small, as this prevents a situation in which $\mathcal{R}$ is sparse, and there is no value in $\mathcal{R}$ that is close to the optimal window size. A natural choice for $\mathcal{R}$ is to use  a sequence of powers where $r_i = \gamma^i$ for some $\gamma > 1$, then $\gamma_m = \gamma_M = 1/\gamma$.
With our analysis, the best guarantees of the algorithm are achieved by using a sequence of powers of $1/(\sqrt{2}-1)^2$ as $\mathcal{R}$.

\item The value of $\beta$ affects the threshold used in our algorithm. Intuitively, the value of $\beta$ is proportional to how much drift the algorithm must observe before stopping, and it affects the sensitivity of our algorithm to detect drift. The optimal value of $\beta$ that minimizes the upper bound of our algorithm is $\beta=\sqrt{2}-1$.
\end{itemize}

In our experiments, we let $\delta = 0.1$ be an arbitrarily small failure probability. We let $\mathcal{R}=\{2^0,2^1,\dots, 2^{19}\}$. We use powers of $2$ rather than powers of $1/(\sqrt{2}-1)^2$ to define $\mathcal{R}$ (see discussion above) for ease as all the resulting window sizes are integers. 

We run an additional experiment
to see the effect of the value $\beta$ on the results of our adaptive algorithm.
We report the accuracy and the F1 score of our adaptive algorithm for the Basketball dataset (Figure~\ref{figure:basketball_acc+F1_vs_beta}) and for the Tennis Rally dataset (Figure~\ref{figure:tennis_acc+F1_vs_beta}). We observe that small values of $\beta$ provide similar results in both datasets, and the value $\beta = \sqrt{2}-1$ is also a good choice. Interestingly, we can notice that for the ``permute'' setup, higher values of $\beta$ lead to a decrease in performance for both datasets. This is most likely due to the fact that for higher values of $\beta$, the algorithm needs to observe a larger magnitude of drift in order to choose a smaller window size. For the ``permute'', it is necessary to react to the introduced drift to obtain better performance, hence smaller values of $\beta$ are more competitive. In our experiments, we choose $\beta = 0.1$ as an arbitrarily small value. 

\begin{figure*}[h!]
\includegraphics[width=.49\textwidth]{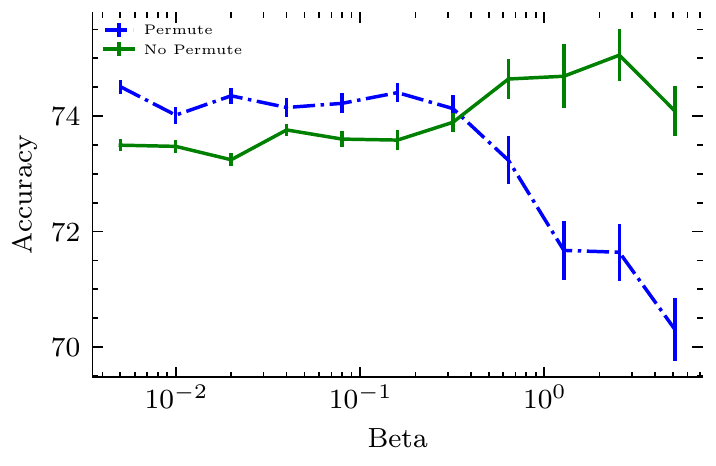}
\hfill
\includegraphics[width=.49\textwidth]{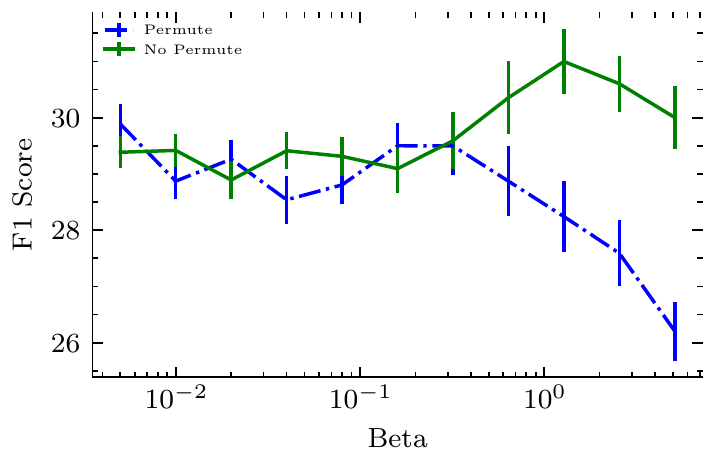}
\caption{Accuracy (left) and F1 score (right) of our adaptive algorithm for the Basketball dataset varying the value of $\beta$. The reported results are an average over $30$ runs.}
\label{figure:basketball_acc+F1_vs_beta}
\end{figure*}

\begin{figure*}[h!]
\includegraphics[width=.49\textwidth]{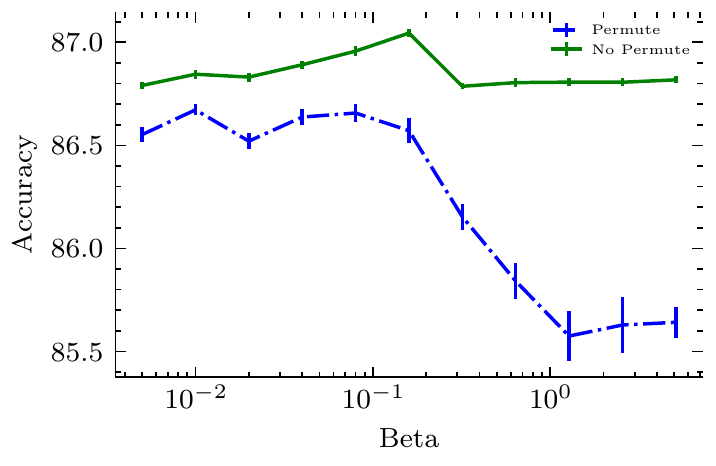}
\hfill
\includegraphics[width=.49\textwidth]{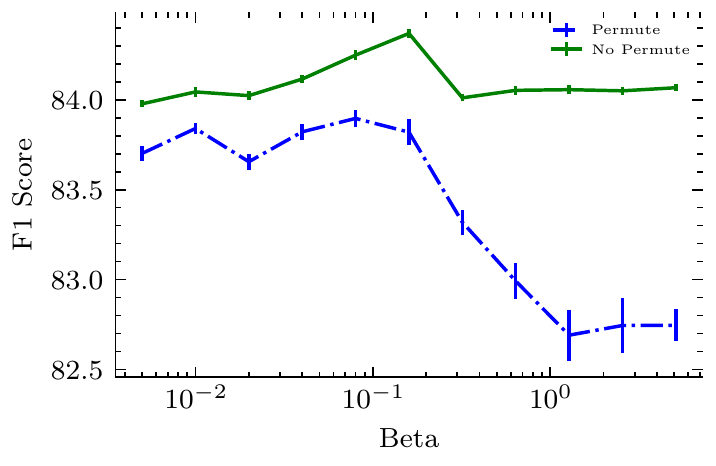}
\caption{Accuracy (left) and F1 score (right) of our adaptive algorithm for the Tennis Rally dataset varying the value of $\beta$. The reported results are an average over $30$ runs.}
\label{figure:tennis_acc+F1_vs_beta}
\end{figure*}

\newpage
\section{Deferred Proofs}
\label{sec:deferred-proof}
In this section, we prove our paper's main theorem (\Cref{thm:main}). We establish the following broader result, from which \Cref{thm:main} follows directly as a corollary. Compared to \Cref{thm:main}, the following theorem explicits the constants in the upper bound, and it allows to consider an arbitrary choice of window sizes $\mathcal{R} = \{ r_1,\ldots,r_m\}$ such that $r_1 < \ldots < r_m$.

\begin{theorem}
\label{thm:main-2}
Let Assumptions~\ref{assu:sample-independents}, \ref{assu:errors-independents} and \ref{assu:bias} hold. Let $\delta \in (0,1)$ and $\beta >0$. Let $\mathcal{R} = \{ r_1, \ldots, r_m \}$.  Assume $n \geq 3$. If we run Algorithm~\ref{alg:finalalgorithm} at time $t \geq r_1$, then with probability at least $1-\delta$ it provides an estimate $\hat{\bm{p}} = (\hat{p}_1,\ldots,\hat{p}_n)$ such that
\begin{align*}
   \left\lVert \bm{p}(t) - \hat{\bm{p}} \right\rVert_{\infty} \leq \frac{5\Phi_{\mathcal{R},\beta}}{2\tau^2}  \min_{r \in [r_1,\min(t,r_m)]} \Bigg(  \frac{A_{\delta,n,m}}{\sqrt{r}} +12\sum_{k=t-r+1}^{r-1} \lVert \bm{p}(k) - \bm{p}(k+1) \rVert_{\infty}\Bigg) 
\end{align*}
where $A_{\delta,n,m} \doteq \sqrt{ 2 \ln[(2m-1)\cdot n(n-1)/\delta]}$, and $
    \Phi_{\mathcal{R},\beta} = 1+ \max\left\{ \frac{2\beta + 2}{\gamma_m(1-\gamma_M)},   \frac{2\beta + 2}{\beta(1-\gamma_M)} \right\}$, with $\gamma_M = \max \sqrt{r_k/r_{k+1}}$ and $\gamma_m = \min\sqrt{r_k/r_{k+1}}$.  
\end{theorem}
A more detailed discussion on the hyper-parameters $\beta$ and $\mathcal{R}$ is provided in \Cref{subsec:hyper-parameters}.

We outline the proof:
\begin{enumerate}
    \item We show that the estimation of the correlation matrix at time $t$ by using the previous $r$ samples can be decomposed into the sum of two error components: \emph{a statistical error} and a \emph{drift error} (\Cref{apx:error-decomposition}).
    \item In order to bound the statistical error, we use standard concentration inequalities to show that the estimation of the correlation matrix obtained by using the $r$ previous samples is close to its expected value with error $O(1/\sqrt{r})$ with high-probability (\Cref{apx:ub-statistical}).
    \item In order to upper bound the drift error, we show an inequality that relates the drift in correlation matrices over time with the drift of the accuracies of the weak labelers (\Cref{apx:ub-drift}).
    \item We use the previous results to show the trade-off between statistical error and drift error depicted in \Cref{lemma:error-decomposition} (\Cref{apx:lemma-proof}).
    \item We prove \Cref{thm:main-2}: we show how to dynamically select the window size in order to optimize the above trade-off (\Cref{apx:dynamic-window}).
\end{enumerate}

\subsection{Error Decomposition}
\label{apx:error-decomposition}
We define the   average correlation matrix over the previous $r$ samples as
\begin{align*}
\bm{C}^{[r]}(t) = \frac{1}{r} \sum_{k=t-r+1}^t \bm{C}(k) \enspace .
\end{align*}
We show the following error decomposition in the upper bound of the error of estimating the matrix $\bm{C}(t)$ by using the empirical matrix $\hat{\bm{C}}^{[r]}(t)$ induced by the previous $r$ samples.
\begin{prop}
\label{prop-apx:error-decomposition}
For any $1 \leq r \leq t$, we have that
\begin{align}
\label{eq:apx-decomposition}
    \left\lVert \bm{C}(t) - \hat{\bm{C}}^{[r]}(t)\right\rVert_{\infty} \leq \underbrace{ \left\lVert \bm{C}^{[r]}(t) - \hat{\bm{C}}^{[r]}(t)\right\rVert_{\infty}}_{\text{statistical error}} + \underbrace{\sum_{i=t-r+1}^{t-1} \left\lVert \bm{C}(i) - \bm{C}(i+1)\right\rVert_{\infty}}_{\text{drift error}} \enspace .
\end{align}
\end{prop}
\begin{proof}
We use the triangle inequality, and obtain that:
\begin{align}
\label{eq-apx:tmp123}
    \left\lVert \bm{C}(t) - \hat{\bm{C}}^{[r]}(t)\right\rVert_{\infty} &= \left\lVert \bm{C}(t) - \bm{C}^{[r]}(t)+ \bm{C}^{[r]}(t) - \hat{\bm{C}}^{[r]}(t)\right\rVert_{\infty} \nonumber \\
    & \leq \left\lVert \bm{C}^{[r]}(t) - \hat{\bm{C}}^{[r]}(t)\right\rVert_{\infty}+ \left\lVert \bm{C}(t) - \bm{C}^{[r]}(t)\right\rVert_{\infty}\enspace .
\end{align}
Again, by using the triangle inequality, we obtain the following chain of inequalities
\begin{align}
\label{eq:useful-drift-inequality}
    \left\lVert \bm{C}(t) - \bm{C}^{[r]}(t)\right\rVert_{\infty} \leq \frac{1}{r}\sum_{i=t-r+1}^t \left\lVert \bm{C}(t) - \bm{C}(i)\right\rVert_{\infty} &\leq \sup_{t-r+1 \leq i \leq t} \lVert \bm{C}(t) - \bm{C}(i)\rVert_{\infty} \nonumber \\
    &\leq \sum_{i=t-r+1}^{t-1} \lVert  \bm{C}(i+1) - \bm{C}(i) \rVert_{\infty} \enspace .
\end{align}
By plugging the above inequality into \eqref{eq-apx:tmp123}, we obtain the statement.

\end{proof}

Observe that by definition, we have the following relation $\Exp \hat{\bm{C}}^{[r]}(t) = \bm{C}^{[r]}(t)$. The statistical error term describes how much the empirical estimation deviate to its expectation, i.e., it is equal to
\begin{align*}
    \left\lVert \hat{\bm{C}}^{[r]}(t) - \Exp \hat{\bm{C}}^{[r]}(t)\right\rVert_{\infty} \enspace .
\end{align*}
This error is related to the variance of $\hat{\bm{C}}^{[r]}$, and we will use a concentration inequality to provide an upper bound to this term (\Cref{apx:ub-statistical}).

The drift error term describes the estimation error due to a change in the accuracy of the weak labelers, and indeed it is equal to $0$ if no change occurs. In \Cref{apx:ub-drift}, we will show how to analytically relate this error to the drift of the weak labelers' accuracies.

\subsection{Upper Bound to the Statistical Error}
\label{apx:ub-statistical}
In this subsection, our main goal is to provide an upper bound to the statistical error term
\begin{align}
\label{eq:apx-tmp}
        \left\lVert \hat{\bm{C}}^{[r]}(t) - \Exp \hat{\bm{C}}^{[r]}(t)\right\rVert_{\infty} \enspace  = \left\lVert \bm{C}^{[r]}(t) - \hat{\bm{C}}^{[r]}(t)\right\rVert_{\infty}
\end{align}
by using a concentration inequality. The following result immediately follows by using McDiarmid's inequality.
\begin{prop}
\label{prop-apx:statistical-error}
Consider a pair of indexes $(i,j) \in \{1,\ldots,n\}^2$. Let $\delta > 0$. With probability at least $1-\delta$, it holds
\begin{align*}
    \left| {C}_{ij}^{[r]}(t) - \hat{{C}}_{ij}^{[r]}(t)\right| \leq \sqrt{\frac{2\ln(2/\delta)}{r}} \enspace .
\end{align*}
\begin{proof}  
Let $f(X_{t-r+1}, \ldots, X_t) = \hat{C}^{[r]}_{ij}(t)$. By definition of $\bm{C}(\cdot)$, it is easy to verify that
\begin{align*}
    \Exp f(X_{t-r+1}, \ldots, X_t) = \frac{1}{r}\sum_{k=t-r+1}^t C_{ij}(k) = C_{ij}^{[r]}(t) \enspace .
\end{align*}
Since each change of a single variable can change the value of $f$ by at most $2/r$, we can use McDiarmid's inequality, and obtain that with probability at least $1-\delta$, it holds that 
\begin{align}
\label{eq-apx:union-bound-statistical}
    \left| f- \Exp f \right|\leq \left| C^{[r]}_{ij}(t) - \hat{C}_{ij}^{[r]}(t)\right| \leq \sqrt{\frac{2\ln(2/\delta)}{r}} \enspace .
\end{align}
\end{proof}
\end{prop}
An upper bound to the statistical error term \eqref{eq:apx-tmp} immediately follows by using the above proposition and taking an union bound over all possible indexes $i,j$. Since the matrices are symmetric, and the diagonal is always equal to $1$, it is sufficient to take an union bound over only $n(n-1)/2$ choices of those indexes. Thus, with probability at least $1-\delta$, it holds that
\begin{align}
\label{eq-apx:union-bound}
    \left\lVert \bm{C}^{[r]}(t) - \hat{\bm{C}}^{[r]}(t)\right\rVert_{\infty} \leq \sqrt{\frac{2\ln(n(n-1)/\delta)}{r}} \enspace .
\end{align}

For our algorithm, we will also need to show an upper bound to the error of estimating the difference between two correlation matrices with different window sizes. Since this result follows with a similar argument of \Cref{prop-apx:statistical-error}, we report it here.

\begin{prop}
\label{prop-apx:statistical-error-difference}
Consider a pair of indexes $(i,j) \in \{1,\ldots,n\}^2$. Let $\delta > 0$, and let $r,r'$ be two integers such that $1 \leq r < r' \leq t$. With probability at least $1-\delta$, it holds
\begin{align*}
    \left| \hat{C}_{ij}^{[r]}(t) - \hat{{C}}_{ij}^{[r']}(t) - C_{ij}^{[r]} + C_{ij}^{[r']}\right| \leq \sqrt{ \frac{2\ln(2/\delta)\left(1-r/r'\right)}{r}} \enspace .
\end{align*}
\end{prop}
\begin{proof}
Let $f(X_{t-r'+1}, \ldots, X_t) = \hat{C}_{ij}^{[r]}(t) - \hat{C}_{ij}^{[r']}(t)$, and observe that \begin{align*}
    |f - \Exp f| =\left| C_{ij}^{[r]}(t) - C_{ij}^{[r']}(t) - \hat{C}_{ij}^{[r]}(t) + \hat{C}_{ij}^{[r']}(t)\right| \enspace .
\end{align*}
The function $f$ is equivalent to
\begin{align*}
    f(X_{t-r'+1}, \ldots, X_t) &= \hat{C}_{ij}^{[r]}(t) - \hat{C}_{ij}^{[r']}(t) \\
    &= \sum_{u=t-r+1}^{t} \left( \frac{1}{r} - \frac{1}{r'} \right) \ell_i(X_u)\ell_j(X_u) - \sum_{u=t-r'+1}^{t-r} \frac{1}{r'}\ell_i(X_u)\ell_j(X_u) \enspace .
\end{align*}
Thus, if we change the variable  $X_u$ with $t-r+1 \leq u \leq t$, the value of $f$ can change by at most $2\left( \frac{1}{r} - \frac{1}{r'} \right)$, and if we change the variable $X_u$ with $t-r'+1 \leq u \leq t-r$, the value of $f$ can change by at most $2/r'$. We can use McDiarmid's inequality, and obtain that with probability at least $1-\delta$, it holds that:
\begin{align}
      &\left| C_{ij}^{[r]}(t) - C_{ij}^{[r']}(t) - \hat{C}_{ij}^{[r]}(t) + \hat{C}_{ij}^{[r']}(t) \right| \nonumber \\
      &\leq \sqrt{\frac{\ln(2/\delta)}{2}} \sqrt{\sum_{u=t-r+1}^{t}4\left( \frac{1}{r} - \frac{1}{r'}\right)^2 + \sum_{u=t-r'+1}^{t-r} \frac{4}{r'^2}} \nonumber \\
      &= \sqrt{2\ln(2/\delta)} \sqrt{r\left( \frac{1}{r} - \frac{1}{r'}\right)^2 + (r'-r) \frac{1}{r'^2}} \nonumber \\ &=\sqrt{2\ln(2/\delta)}\sqrt{\frac{(r'-r)^2 + r(r'-r)}{rr'^2}} \nonumber \\
      &= \sqrt{2\ln(2/\delta)}\sqrt{\frac{r'(r'-r)}{r r'^2}} \nonumber \\
      &= \sqrt{  \frac{ 2\ln(2/\delta)(1-r/r')}{r}} \nonumber
\end{align}
\end{proof}

We end this subsection by providing a result that we quote during the explanation of the algorithm. While this result is not necessary to prove \Cref{thm:main}, it has a similar flavour than the previous proposition, and we report its proof here.
\begin{prop}
\label{prop-apx:intuition}
Let $1 \leq r \leq r' \leq t$. If $D_1 = \ldots = D_t$, then for any pair of indexes $i,j \in \{1,\ldots,\}^n$, it holds
\begin{align*}
    \Exp| \hat{C}_{ij}^{[r]}(t) - \hat{C}_{ij}^{[r']}(t)| \leq \sqrt{\frac{1}{r} - \frac{1}{r'}} \enspace .
\end{align*}
\end{prop}
\begin{proof}
For $i=j$ the statement is trivially true as the difference is $0$. Let $i \neq j$. Consider the random variables $Z_k = \ell_i(X_{t-k+1})\ell_j(X_{t-k+1})$ for $1 \leq k \leq r'$. By assumption, the random variables are independent, and $Z_k \in [-1,1]$. By using the definition of $\hat{\bm{C}}^{[r]}$, we have:
\begin{align*}
    \Exp| \hat{C}_{ij}^{[r]}(t) - \hat{C}_{ij}^{[r']}(t)| &= \Exp \left| \frac{1}{r}\sum_{k=1}^r Z_k - \frac{1}{r'}\sum_{k=1}^{r'} Z_k \right| \\
    &\leq \sqrt{\mathbb{V}\left(\frac{1}{r}\sum_{k=1}^r Z_k - \frac{1}{r'}\sum_{k=1}^{r'} Z_k\right)} \enspace ,
\end{align*}
where in the last step we used Jensen's inequality. Now, we have that:
\begin{align*}
    \mathbb{V}\left(\frac{1}{r}\sum_{k=1}^r Z_k - \frac{1}{r'}\sum_{k=1}^{r'} Z_k\right) &=   \mathbb{V}\left(\left(\frac{1}{r} - \frac{1}{r'}\right)\sum_{k=1}^r Z_k - \frac{1}{r'}\sum_{k=r+1}^{r'} Z_k\right)\\
    &= \left[\left( \frac{1}{r} - \frac{1}{r'} \right)^2r + \frac{1}{r'^2}(r'-r) \right] \mathbb{V}(Z_1) \\
    &= \frac{r' - r}{rr'}\mathbb{V}(Z_1) \enspace .
\end{align*}
Since $Z_1 \in [-1,1]$, by Popoviciu's inequality we have that $\mathbb{V}(Z_1) \leq 1$. Hence, we can conclude that 
\begin{align*}
    \Exp| \hat{C}_{ij}^{[r]}(t) - \hat{C}_{ij}^{[r']}(t)| \leq \sqrt{\frac{1}{r} - \frac{1}{r'}} \enspace .
\end{align*}
\end{proof}

\subsection{Upper Bound to the Drift Error}
\label{apx:ub-drift}
In this subsection, we show how to provide an upper bound to the drift error term
\begin{align*}
\sum_{i=t-r+1}^{t-1} \left\lVert \bm{C}(i) - \bm{C}(i+1)\right\rVert_{\infty}
\end{align*}
as a function of the variation in the weak labelers' accuracies $\bm{p}(t-r+1), \ldots, \bm{p}(t)$. Intuitively, the correlation matrix does not change if the weak labelers' accuracies are the same, and a bounded drift in the those accuracies  also implies a small variation in the correlation matrix. This is formalized in the following proposition.
\begin{prop}
\label{apx-prop:drift-covariance-accuracy}
For any $1 \leq k \leq t-1$, the following inequality holds
\begin{align*}
\left\lVert \bm{C}(k) - \bm{C}(k+1)\right\rVert_{\infty} \leq 12 \lVert \bm{p}(k) - \bm{p}(k+1)\rVert_{\infty}
\end{align*}
\end{prop}
\begin{proof}
Consider coordinates $i,j$ such that $i \neq j$.
By definition of $C_{i,j}(k)$, we have that
\begin{align*}
    C_{i,j}(k) = \Exp_{X \sim D_k} \left[ \ell_i(X) \cdot \ell_j(X) \right] \enspace .
\end{align*}
We have that $\ell_i(X) \cdot \ell_j(X)$ is equal to $1$ if and only if $\ell_i$ and $\ell_j$ are both either correct or incorrect, and is equal to $-1$ otherwise. By using the definition of $p_i(k)$ and \Cref{assu:errors-independents}, we have that 
\begin{align*}
C_{i,j}(k)= \Exp_{X \sim D_k} \left[ \ell_i(X) \cdot \ell_j(X) \right] &=  p_i(k) p_j(k) + (1-p_i(k))(1-p_j(k)) \nonumber \\ & \hspace{15pt}- p_i(k)(1-p_j(k)) - p_j(k)(1-p_i(k)) \nonumber \\
&= 4p_i(k)p_j(k) - 2p_i(k) - 2p_j(k) + 1 \nonumber \enspace .
\end{align*}
Hence, we have that
\begin{align}
\label{eq:tmp2}
    | C_{i,j}(k) - C_{i,j}(k+1)| &= \big| 4p_i(k)p_j(k) - 4p_i(k+1)p_j(k+1) \nonumber \\ &\hspace{25pt}+ 2p_i(k+1) + 2p_j(k+1)- 2p_i(k) - 2p_j(k)\big| \nonumber \\
    &\leq 4\big| p_i(k)p_j(k) - p_i(k+1)p_j(k+1) \big| \nonumber \\
     &\hspace{25pt}+2\big| p_i(k+1) - p_i(k) \big| + 2\big| p_j(k+1) - p_j(k) \big| \enspace .
\end{align}
where the first inequality follows from the triangle inequality. For ease of notation, let $\rho_k = \lVert \bm{p}(k+1) - \bm{p}(k)\rVert_{\infty}$.  We have that
\begin{align*}
p_i(k)p_j(k) & = p_i(k)(p_j(k) + p_j(k+1) - p_j(k+1))\\
&\leq p_i(k)p_j(k+1) + p_i(k)|p_j(k) - p_j(k+1)| \\
&\leq p_i(k)p_j(k+1) + \rho_k  \\
&= (p_i(k)-p_i(k+1)+p_i(k+1))p_j(k+1) + \rho_k  \\
&\leq p_i(k+1)p_j(k+1) + 2 \rho_k \enspace .
\end{align*}
which implies that $p_i(k)p_j(k)-p_i(k+1)p_j(k+1) \leq 2\rho_k$. Similarly, we can show that $p_i(k+1)p_j(k+1)-p_i(k)p_j(k) \leq 2\rho_k$, hence we have that $|p_i(k+1)p_j(k+1)-p_i(k)p_j(k)| \leq 2\rho_k$. 
By using this inequality in \eqref{eq:tmp2}, we obtain $| C_{i,j}(k) - C_{i,j}(k+1)|\leq 12 \rho_k$. The statement follows by substituting the definition of $\rho_k$.
\end{proof}

\subsection{Proof of \Cref{lemma:error-decomposition}}
\label{apx:lemma-proof}
\begin{proof}
With \Cref{prop-apx:error-decomposition}, we have that:
\begin{align*}
\left\lVert \bm{C}(t) - \hat{\bm{C}}^{[r]}(t)\right\rVert_{\infty} \leq  \left\lVert \bm{C}^{[r]}(t) - \hat{\bm{C}}^{[r]}(t)\right\rVert_{\infty} + \sum_{i=t-r+1}^{t-1} \left\lVert \bm{C}(i) - \bm{C}(i+1)\right\rVert_{\infty} \enspace .
\end{align*}
In order to conclude the result, we upper bound each term of the right-hand side of the above inequality individually. We use \Cref{prop-apx:statistical-error} and take an union bound over $n(n-1)/2$ coordinates (see also \eqref{eq-apx:union-bound}), and we obtain that with probability at least $1-\delta$, it holds
\begin{align*}
    \left\lVert \bm{C}^{[r]}(t) - \hat{\bm{C}}^{[r]}(t)\right\rVert_{\infty} \leq \sqrt{\frac{2\ln(n(n-1)/\delta)}{r}}
\end{align*}
\Cref{apx-prop:drift-covariance-accuracy} yields the following upper bound:
\begin{align*}
\sum_{i=t-r+1}^{t-1} \left\lVert \bm{C}(i) - \bm{C}(i+1)\right\rVert_{\infty} \leq 12\sum_{i=t-r+1}^{t-1} \left\lVert \bm{p}(i) - \bm{p}(i+1)\right\rVert_{\infty} \enspace
\end{align*}

\end{proof}

\subsection{Dynamic Selection of the Window Size (\Cref{thm:main-2})}
\label{apx:dynamic-window}
In this subsection, we show how to adaptively choose the number of past samples that minimizes a trade-off between the statistical error and the drift error.  
As a prerequisite, our algorithm requires that the used empirical quantities provide a good approximation of their estimated expectations. If this is not the case, we simply assume that our algorithm fails, and this happens with probability $\leq \delta$. The next corollary formalizes this required guarantee on the estimation.
We remind the definition of the value $A_{\delta,n,m}$ in the statement of \Cref{thm:main-2}:
\begin{align*}
A_{\delta,n,m} = \sqrt{2 \ln[(2m-1)\cdot n(n-1)/\delta]}
\end{align*}
\begin{coro}
Let $\delta > 0$
\label{coro:required-event}
Let $\mathcal{R} = \{r_1,\ldots,r_m\}$. 
With probability at least $1-\delta$, it holds:
\begin{align*}
    &\lVert \bm{C}(t) - \hat{\bm{C}}^{[r_k]}(t)  \rVert_{\infty} \leq \frac{A_{\delta,n,m}}{\sqrt{r_k}} + \left\lVert \bm{C}(t) - \bm{C}^{[r]}(t)\right\rVert_{\infty}   &\forall k \leq m \\
    &\lVert \bm{C}^{[r_k]}(t) - \bm{C}^{[r_{k+1}]}(t) - \hat{\bm{C}}^{[r_k]}(t) + \hat{\bm{C}}^{[r_{k+1}]}(t)  \rVert_{\infty} \leq A_{\delta,n,m}\sqrt{ \frac{1-r_k/r_{k+1}}{r_k}}   &\forall k \leq m-1
\end{align*}
\end{coro}
\begin{proof}
By using the triangle inequality, we have that
\begin{align*}
    \lVert \bm{C}(t) - \hat{\bm{C}}^{[r_k]}(t)  \rVert_{\infty} \leq \left\lVert \bm{C}^{[r_k]}(t) - \hat{\bm{C}}^{[r_k]}(t)\right\rVert_{\infty} + \left\lVert \bm{C}(t) - \bm{C}^{[r]}(t)\right\rVert_{\infty}
\end{align*}
We use \Cref{prop-apx:statistical-error} and \Cref{prop-apx:statistical-error-difference} to upper bound respectively $ \lVert \bm{C}(t) - \hat{\bm{C}}^{[r_k]}(t)  \rVert_{\infty}$ and  $\lVert \bm{C}^{[r_k]}(t) - \bm{C}^{[r_{k+1}]}(t) - \hat{\bm{C}}^{[r_k]}(t) + \hat{\bm{C}}^{[r_{k+1}]}(t)  \rVert_{\infty}$. Those propositions provide a guarantee for a single choice of window sizes and coordinates: we take an union bound over $n(n-1)/2$ choice of coordinates and $m+(m-1)$ different choice of window sizes, hence we take an union bound over $(2m-1)n(n-1)/2$ events. The statement immediately follows by an inspection of the value $A_{\delta,n,m}$.
\end{proof}

Throughout this subsection, we assume that the event of \Cref{coro:required-event} holds, otherwise our algorithm fails (with probability $\leq \delta$). Let $\beta$, $\gamma_m$ and $\gamma_M$ be defined as in \Cref{thm:main-2}.
We define the following function:
\begin{align*}
    \mathcal{B}(r) = \frac{A_{\delta,n,m}}{\sqrt{r}}\frac{2\beta+2}{1-\gamma_M} + \left\lVert \bm{C}(t) - \bm{C}^{[r]}(t) \right\rVert_{\infty}
\end{align*}
The value $\mathcal{B}(r)$ is the upper bound that our algorithm guarantees to $\lVert C(t) - \hat{C}^{[r]}(t) \rVert_{\infty}$ using any value $r \in \mathcal{R}$. In fact, we have that
\begin{align*}
    \mathcal{B}(r) &= \frac{A_{\delta,n,m}}{\sqrt{r}}\frac{2\beta+2}{1-\gamma_M} + \left\lVert \bm{C}(t) - \bm{C}^{[r]}(t) \right\rVert_{\infty} \\
    &\geq  \frac{A_{\delta,n,m}}{\sqrt{r}} + \left\lVert \bm{C}(t) - \bm{C}^{[r]}(t) \right\rVert_{\infty} \\
    &\geq  \lVert \bm{C}(t) - \hat{\bm{C}}^{[r]}(t) \rVert_{\infty} \hspace{150pt} \forall r \in \mathcal{R} \enspace ,
\end{align*}
where the last inequality follows from \Cref{coro:required-event}.  For any value $k \leq m-1$, also let
\begin{align*}
    \mathcal{T}(k) \doteq 2 \beta A_{\delta,n,m} \sqrt{\frac{1}{r_k}} + A_{\delta,n,m} \sqrt{\frac{1- r_{k}/r_{k+1}}{r_k}} \enspace ,
\end{align*}
and observe that this is the quantity used as a threshold in Line~5 of the algorithm at iteration $k$.

The proof of \Cref{thm:main-2} revolves around the following two Propositions~\ref{prop:threshold-less} and~\ref{prop:threshold-more}.
\begin{compactenum}
    \item We guarantee that if $\lVert \hat{\bm{C}}^{[r_{k+1}]}(t)  - \hat{\bm{C}}^{[r_k]}(t) \rVert_{\infty}$ is smaller than the threshold $\mathcal{T}(k)$, then a negligeable drift occured, and the upper bound $\mathcal{B}(r_{k+1})$ is smaller than $\mathcal{B}(r_k)$ (\Cref{prop:threshold-less}) In this case, we can keep iterating.
    \item On the other hand, if $\lVert \hat{\bm{C}}^{[r_{k+1}]}(t)  - \hat{\bm{C}}^{[r_k]}(t) \rVert_{\infty}$ is greater than the threshold $\mathcal{T}(k)$, a sizeable drift occurred, and we can provide a lower bound on the drift error (\Cref{prop:threshold-more}). In this case, we can stop iterating and return the current window size $r_k$.
\end{compactenum}

We prove those two propositions.
\begin{prop}
\label{prop:threshold-less}
Let the event of \Cref{coro:required-event} hold. Then, for any $1 \leq k \leq m-1$
\begin{align*}
    \lVert \hat{\bm{C}}^{[r_k]} - \hat{\bm{C}}^{[r_{k+1}]} \rVert_{\infty} \leq \mathcal{T}(k) \Longrightarrow \mathcal{B}(r_{k+1}) \leq \mathcal{B}(r_k)
\end{align*}
\end{prop}
\begin{proof}
We have that 
\begin{align}
\label{eq:tmp123}
    \mathcal{B}(r_{k+1}) - \mathcal{B}(r_k) =  {A_{\delta,n,m}}\frac{2\beta+2}{1-\gamma_M}\left[ \sqrt{ \frac{1}{r_{k+1}}} -   \sqrt{ \frac{1}{r_{k}}}\right] + \lVert \bm{C}(t) - \bm{C}^{[r_{k+1}]}(t) \rVert_{\infty} - \lVert \bm{C}(t) - \bm{C}^{[r_{k}]}(t) \rVert_{\infty}
\end{align}
We can obtain the following upper 
\begin{align*}
    &\lVert \bm{C}(t) - \bm{C}^{[r_{k+1}]}(t) \rVert_{\infty} - \lVert \bm{C}(t) - \bm{C}^{[r_{k}]}(t) \rVert_{\infty} \\ \leq \hspace{5pt} &  \lVert \bm{C}^{[r_{k+1}]}(t) - \bm{C}^{[r_k]}(t) \rVert_{\infty} \\ 
    = \hspace{5pt} & \lVert \hat{\bm{C}}^{[r_k]}(t) - \hat{\bm{C}}^{[r_{k+1}]}(t) - \hat{\bm{C}}^{[r_k]}(t) + \hat{\bm{C}}^{[r_{k+1}]}(t) +  \bm{C}^{[r_k]}(t) - \bm{C}^{[r_{k+1}]}(t)\rVert_{\infty}\\ 
    \leq  \hspace{5pt} & \lVert \hat{\bm{C}}^{[r_k]}(t) - \hat{\bm{C}}^{[r_{k+1}]}(t) \rVert_{\infty} + \lVert -  \hat{\bm{C}}^{[r_k]}(t) + \hat{\bm{C}}^{[r_{k+1}]}(t)  +   \bm{C}^{[r_k]}(t) - \bm{C}^{[r_{k+1}]}(t)\rVert_{\infty} \\
    \leq \hspace{5pt} & \mathcal{T}(k) + A_{\delta,n,m}\sqrt{ \frac{1-r_k/r_{k+1}}{r_k}}
\end{align*}
where the first two inequalities are due to the triangle inequality, and the last inequality is due to the assumption of the proposition statement and \Cref{coro:required-event}. By plugging the above inequality in \eqref{eq:tmp123} and using the definition of $\mathcal{T}(k)$, we obtain 
\begin{align}
\label{eq:tmp124}
     \mathcal{B}(r_{k+1}) - \mathcal{B}(r_k) \leq \frac{A_{\delta,n,m}}{\sqrt{r_k}}\left[\frac{2\beta+2}{1-\gamma_M} \sqrt{\frac{r_{k}}{r_{k+1}}} - \frac{2\beta+2}{1-\gamma_M} + 2\beta +2\sqrt{1-r_k/r_{k+1}}  \right] 
\end{align}
We have that
\begin{align*}
    &\left[\frac{2\beta+2}{1-\gamma_M} \sqrt{\frac{r_{k}}{r_{k+1}}} - \frac{2\beta+2}{1-\gamma_M} + 2\beta +2\sqrt{1-r_k/r_{k+1}}  \right] \\ \leq \hspace{5pt}  &\left[ \frac{2\beta+2}{1-\gamma_M} \gamma_M - \frac{2\beta+2}{1-\gamma_M} + 2\beta +2 \right] \\
    =  \hspace{5pt} & 0
\end{align*}
By using this inequality in \eqref{eq:tmp124}, we finally obtain that $\mathcal{B}(r_{k+1}) - \mathcal{B}(r_k) \leq 0$.
\end{proof}
This proposition guarantees that every time the If of Line~5 of the algorithm is true, then $\mathcal{B}(r_{k+1})$ is an upper bound at least as good as $\mathcal{B}(r_k)$. Conversely, the next proposition shows that when the If of Line~5 is false, a drift must have occurred.
\begin{prop}
\label{prop:threshold-more}
Let the event of \Cref{coro:required-event} hold. Then, for any $1 \leq k \leq m-1$
\begin{align*}
    \left\lVert \hat{\bm{C}}^{[r_k]} - \hat{\bm{C}}^{[r_{k+1}]} \right\rVert_{\infty} >\mathcal{T}(k) \Longrightarrow \sum_{u=t-r_{k+1}+1}^{t-1} \lVert \bm{C}(u) - \bm{C}(u+1) \rVert_{\infty} > A_{\delta,n,m}\cdot \beta/\sqrt{r_k}
\end{align*}
\begin{proof}
We have that:
\begin{align*}
      \lVert \hat{\bm{C}}^{[r_k]} - \hat{\bm{C}}^{[r_{k+1}]} \rVert_{\infty} &\leq A_{\delta,n,m}\sqrt{\frac{1-r_{k}/r_{k+1}}{r_k}} + \lVert \bm{C}^{[r_k]}(t) - \bm{C}^{[r_{k+1}]}(t) \rVert_{\infty} \\
      &\leq A_{\delta,n,m}\sqrt{\frac{1-r_{k}/r_{k+1}}{r_k}} + \lVert C(t) - C^{[r_{k+1}]}(t) \rVert_{\infty} + \lVert C^{[r_k]}(t) - C(t) \rVert_{\infty}
\end{align*}
where the first inequality is due to \Cref{coro:required-event}, and the second inequality is due to the triangle inequality. By using  \eqref{eq:useful-drift-inequality}, we can show that
\begin{align*}
    \lVert \bm{C}(t) - \bm{C}^{[r_{k+1}]}(t) \rVert_{\infty} + \lVert \bm{C}^{[r_k]}(t) - \bm{C}(t) \rVert_{\infty} \leq 2\sum_{u=t-r_{k+1}+1}^{t-1}\lVert \bm{C}(u) - \bm{C}(u+1)  \rVert_{\infty} \enspace .
\end{align*}
Hence, by using the assumption of the proposition and the definition of $\mathcal{T}(k)$, we obtain the following inequality:
\begin{align*}
    2\sum_{u=t-r_{k+1}+1}^{t-1}\lVert \bm{C}(u) - \bm{C}(u+1)  \rVert_{\infty}  > 2 \beta A_{\delta,n,m}/\sqrt{r_k} \enspace ,
\end{align*}
and the statement immediately follows.
\end{proof}
\end{prop}

In the following Lemma, we use Proposition~\ref{prop:threshold-less} and \ref{prop:threshold-more} to show that the matrix $\hat{\bm{C}}$ of Line~7 of the algorithm provides a good approximation of $\bm{C}(t)$. Theorem~\ref{thm:main-2} immediately follows from this result by using \Cref{prop:covariance-to-expertise}.

\begin{lemma}
\label{lemma:main-result}
Consider the setting of \Cref{thm:main-2}. Let $\hat{C}$ be the matrix defined at Line~7 of the algorithm. With probability at least $1-\delta$, it holds that
\begin{align*}
    \lVert \hat{\bm{C}} - \bm{C}(t)\rVert_{\infty} \leq \Phi_{\mathcal{R},\beta} \min_{r \in [r_1,\min(t,r_m)]} \left(  \frac{A_{\delta,n,m}}{\sqrt{r}} + \sum_{u=t-r+1}^{t-1} \lVert C(u) - C(u+1)\rVert_{\infty}\right)
\end{align*}
\end{lemma}
\begin{proof}
Assume that the event of \Cref{coro:required-event} holds (otherwise we say that our algorithm fails, with probability $\leq \delta$). For ease of notation, let $\nu = (2\beta+2)/(1-\gamma_M)$. Let $\hat{k} \leq m$ be the value such that $\hat{\bm{C}} = \hat{\bm{C}}^{[r_{\hat{k}}]}(t)$. We remind that the algorithm guarantees an upper bound $\mathcal{B}(r_{\hat{k}})$ to the estimation error  $\lVert \hat{\bm{C}} - \bm{C}(t)\rVert_{\infty}$ Let $r^*$ be the integer that minimizes \begin{align*}
    r^* = \mathrm{argmin}_{r \in [r_1, \min(t,r_m)]} \left( \frac{A_{\delta,n,m}}{\sqrt{r}} + \sum_{u=t-r+1}^{t-1} \lVert C(u) - C(u+1)\rVert_{\infty}\right),
\end{align*} 
and let $\mathcal{B}^*$ be the minimum value of the above expression, i.e.
\begin{align*}
    \mathcal{B}^* = \frac{A_{\delta,n,m}}{\sqrt{r^*}} + \sum_{u=t-r^*+1}^{t-1} \lVert C(u) - C(u+1)\rVert_{\infty}
\end{align*}
In order to prove the lemma, it is sufficient to show that $\mathcal{B}(r_{\hat{k}})/\mathcal{B}^* \leq \Phi_{\mathcal{R},\beta}$.

We distinguish two cases: $(a)$ $\hat{k} = m$ or $r^* < r_{\hat{k}+1}$ and $(b)$ $r^* \geq r_{\hat{k}+1}$.
Consider case $(a)$. Let $\tilde{k}$ be the largest integer such that $r_{\tilde{k}} \leq r^*$. By construction, we can observe that $\tilde{k} \leq \hat{k}$. Since the algorithm did not interrupt in the iterations $1,\ldots,\tilde{k}, \ldots, \hat{k}$, we can use \Cref{prop:threshold-less}, to show that $\mathcal{B}(r_{\hat{k}}) \leq \mathcal{B}(r_{\tilde{k}})$. Hence, we have that:
\begin{align*}
    \frac{\mathcal{B}(r_{\hat{k}})}{\mathcal{B}^*} \leq  \frac{\mathcal{B}(r_{\tilde{k}})}{\mathcal{B}^*} &= \frac{\nu \cdot A_{\delta,n,m}/\sqrt{r_{\tilde{k}}} + \left\lVert \bm{C}(t) - \bm{C}^{[r_{\tilde{k}}]}(t)\right\rVert_{\infty}}{A_{\delta,n,m}/\sqrt{r^*} + \sum_{u=t-r^*+1}^{t-1} \lVert \bm{C}(u) - \bm{C}(u+1)\rVert_{\infty}} \\
    & \leq \frac{\nu \cdot A_{\delta,n,m}/\sqrt{r_{\tilde{k}}} + \sum_{u=t-r_{\tilde{k}}+1}^{t-1} \lVert \bm{C}(u) - \bm{C}(u+1)\rVert_{\infty}}{A_{\delta,n,m}/\sqrt{r^*} + \sum_{u=t-r^*+1}^{t-1} \lVert \bm{C}(u) - \bm{C}(u+1)\rVert_{\infty}} \\
    &\leq \nu \sqrt{\frac{r^*}{r_{\tilde{k}}}} + \frac{\sum_{u=t-r_{\tilde{k}}+1}^{t-1} \lVert \bm{C}(u) - \bm{C}(u+1)\rVert_{\infty}}{\sum_{u=t-r^*+1}^{t-1} \lVert \bm{C}(u) - \bm{C}(u+1)\rVert_{\infty}}\\
    &\leq   \nu/\gamma_m + 1
\end{align*}
where the second inequality is due to \eqref{eq:useful-drift-inequality}, and the last inequality is due to the definition of $r_{\tilde{k}}$. We can observe that $1+\nu/\gamma_m  = 1 + \frac{2\beta+2}{\gamma_m(1-\gamma_M)} \leq \Phi_{\mathcal{R},\beta}$ and this concludes the first part of the proof.

We consider case $(b)$. Since the algorithm stopped at iteration $\hat{k} < m$, the If condition of Line~4 is false during this iteration, and due to \Cref{prop:threshold-more}, we have that:
\begin{align}
\label{eq:tmpbella}
    \sum_{u=t-r^*+1}^{t-1} \lVert \bm{C}(u) - \bm{C}(u+1)\rVert_{\infty} \geq \sum_{u=t-r_{\hat{k}+1}+1}^{t-1} \lVert \bm{C}(u) - \bm{C}(u+1)\rVert_{\infty} \geq A_{\delta,n,m}\beta/\sqrt{r_{\hat{k}}} \enspace .
\end{align}
We obtain:
\begin{align*}
    \frac{\mathcal{B}(r_{\hat{k}})}{\mathcal{B}^*} &=\frac{\nu \cdot A_{\delta,n,m}/\sqrt{r_{\hat{k}}} + \left\lVert \bm{C}(t) - \bm{C}^{[r_{\hat{k}}]}(t) \right\rVert_{\infty}}{A_{\delta,n,m}/\sqrt{r^*} + \sum_{u=t-r^*+1}^{t-1} \lVert \bm{C}(u) - \bm{C}(u+1)\rVert_{\infty}} \\
    &\leq \frac{\nu \cdot A_{\delta,n,m}/\sqrt{r_{\hat{k}}} + \sum_{u=t-r_{\hat{k}}+1}^{t-1} \lVert \bm{C}(u) - \bm{C}(u+1)\rVert_{\infty}}{A_{\delta,n,m}/\sqrt{r^*} + \sum_{u=t-r^*+1}^{t-1} \lVert \bm{C}(u) - \bm{C}(u+1)\rVert_{\infty}} \\
    &\leq \frac{\nu \cdot A_{\delta,n,m}/\sqrt{r_{\hat{k}}}}{\sum_{u=t-r^*+1}^{t-1} \lVert \bm{C}(u) - \bm{C}(u+1)\rVert_{\infty}} + \frac{\sum_{u=t-r_{\hat{k}}+1}^{t-1} \lVert \bm{C}(u) - \bm{C}(u+1)\rVert_{\infty}}{\sum_{u=t-r^*+1}^{t-1} \lVert \bm{C}(u) - \bm{C}(u+1)\rVert_{\infty}} \\
    &\leq  \frac{\nu}{\beta}+1 \enspace .
\end{align*}
where in the last inequality we used \eqref{eq:tmpbella} and the fact that $r_{\hat{k}} \leq r^*$. We finally observe that $\nu/\beta+1= 1+\frac{2\beta+2}{(1-\gamma_M)\beta} \leq \Phi_{\mathcal{R},\beta}$, and this concludes the proof.
\end{proof}

We can finally prove \Cref{thm:main-2} as a simple corollary of \Cref{lemma:main-result}.

\begin{proof}[Proof of \Cref{thm:main-2}]
Let $\epsilon$ be the guarantee of \Cref{lemma:main-result}. If we let $\hat{C}$ be the matrix of Line~7 of the algorithm, we have that with probability at least $1-\delta$, it holds:
\begin{align*}
    \left\lVert \hat{\bm{C}} - \bm{C}(t)\right\rVert_{\infty} \leq \epsilon
\end{align*}
Lines 8-12 of the algorithm implement the procedure of \cite{bonald2017minimax} that attains the guarantee of \Cref{prop:covariance-to-expertise}. Hence, we have that:
\begin{align*}
\lVert \hat{\bm{p}} - \bm{p}(t)\rVert_{\infty} &\leq \frac{5\epsilon}{2\tau^2} \\
&\leq  \frac{5\Phi_{\mathcal{R},\beta}}{2\tau^2} \min_{r \in [r_1,\min(t,r_m)]} \left(  \frac{A_{\delta,n,m}}{\sqrt{r}} + \sum_{u=t-r+1}^{t-1} \lVert \bm{C}(u) - \bm{C}(u+1)\rVert_{\infty}\right) \enspace .
\end{align*}
The statement immediately follows by using \Cref{apx-prop:drift-covariance-accuracy}.
\end{proof}

\section{Relaxing the Conditional Independence Assumption}
\label{app:relaxation}
The analysis of our algorithm uses the conditional independence assumption for the error of the weak labelers (\Cref{assu:errors-independents}). In this section, we explore a possible relaxation of this assumption based on previous work that handles the general case where the weak labelers are arbitrarily correlated \cite{mazzetto:aistats21}. We propose a variant of our method that can be adopted in this setting. 

At the time step $t$, consider the random vector 
\begin{align*}
    \bm{a}(t) = \left( \mathbf{1}_{\ell_1(x) = y(x)}, \ldots, \mathbf{1}_{\ell_n(x) = y(x)} \right) \hspace{20pt}  \mbox{ where } x \sim D_t
\end{align*}
Indeed, we have that $\Exp \bm{a}(t) = \bm{p}(t)$. The conditional independence assumption allows us to factorize the distribution of $\bm{a}(t)$, i.e., for any $\bm{a} \in \{0,1\}^n$, we have that:
\begin{align}
\label{eq:factorization}
\Pr( \bm{a}(t) = \bm{a}) =  \prod_{i=1}^n \Pr( \bm{a}(t)_i = \bm{a}_i) \enspace .
\end{align}
The factorization \eqref{eq:factorization} is the property that is exploited by the methods that are based on this conditional independence assumption, and it is also used to define the optimal aggregation rule given by \eqref{eq:w^*}.

If we relax \Cref{assu:errors-independents}, Equation~\eqref{eq:factorization} no longer holds, and in general, the optimal aggregation rule depends on the dependencies between the weak labelers' errors. In particular, the previous work \cite{mazzetto:aistats21} shows that in the worst-case it is not possible to improve upon the most accurate weak labeler, without any further assumption on the errors' dependencies (a clear example is when $\ell_1 = \ldots = \ell_n$). 

The previous work \cite{mazzetto:aistats21} provides a method to determine a subset of weak labelers whose majority vote has provably bounded error, without introducing any assumption on the dependencies between the weak labelers' error. The idea is to leverage the information about the correlation to determine a subset of weak labelers who make errors in different parts of the distribution domain so that their majority vote can increase their aggregate performance. This method relies on the following quantities: $(1)$ the correlation between the weak labelers outputs $\bm{C}(t)$, which can be estimated only using unlabeled data; and $(2)$ the weak labelers' accuracies $\bm{p}(t)$. As we will see, we can use our adaptive method to maintain an estimate of the correlation matrix in a non-stationary setting.

With access to $\bm{p}(t)$ and $\bm{C}(t)$, the worst-case error of the majority vote for the weak labelers $\{ \ell_1,\ldots,\ell_n \}$ at time $t$ can be computed as a solution of the following linear program: 

\begin{align*}
    (*) \hspace{30pt} \max&\sum_{\bm{a} \in \{0,1\}^{n} : \lVert \bm{a}\rVert_1 < n/2} p_{\bm{a}} & \\
    &(a) \hspace{5pt} \sum_{\bm{a} \in \{0,1\}^n : a_i=1 }p_{ \bm{a}} = \bm{p}(t)_i  \hspace{10pt} & \mbox{ for }   i = 1,\ldots,n \\
    &(b) \hspace{5pt} \sum_{\bm{a} \in \{0,1\}^n : a_i = a_j}  p_{ \bm{a}}   = (\bm{C}(t)_{ij}+1)/2 \hspace{5pt} &\mbox{ for } i \neq j \\
    &(c) \hspace{5pt} \sum_{\bm{a}}p_{ \bm{a}} = 1 &\\
    &(d) \hspace{5pt}0 \leq p_{\bm{a}}\leq 1 \hspace{5pt} &\forall \bm{a}
\end{align*}
As described in \cite{mazzetto:aistats21}, the optimization problem (*) can be used as a subroutine to find a subset of weak labelers which has provably a small worst-case error for their majority vote. However, the optimization problem (*) requires access to the unknown quantities $\bm{p}(t)$ and $\bm{C}(t)$. 

In many practical applications, it is realistic to assume that the weak labelers provide a significantly stronger signal than noise, i.e. \Cref{assu:bias} holds for a value of $\tau \in (0,1/2)$ that is far from zero. Then it is possible to replace constraint $(a)$ with $\sum_{\bm{a} \in \{0,1\}^n : a_i=1 }p_{ \bm{a}} \geq \frac{1}{2} + \tau$. In particular, for each weak labeler $\ell_i$, we can use a different value $\tau_i$ to reflect our believe on how much this weak labeler is accurate \cite{arachie:aaai19,arachie:jmlr21}.

We leverage our adaptive method to keep track of the matrix $\bm{C}(t)$. Even in a non-stationary setting, where we only have access to a single sample from $P_t$, we can estimate $\bm{C}(t)$ from the samples $X_1,\ldots,X_t$ using \Cref{alg:finalalgorithm}. In fact, in our analysis we show that the matrix $\bm{\hat{C}}$ computed in Line~7 of our algorithm achieves a near-optimal trade-off between statistical error and drift error (\Cref{lemma:main-result}), providing an empirical estimate of $\bm{C}(t)$ in a drift setting.

Thus, we can obtain an algorithm for the case where the weak labelers' errors are not conditionally independent, and they can be arbitrarily correlated. At time $t$, we use \Cref{alg:finalalgorithm} to estimate $\bm{C}(t)$ with $\bm{\hat{C}}$. Then, we use the information on the correlation matrix to identify a subset of weak labelers whose majority vote is accurate. In particular, we use the algorithm described in \cite{mazzetto:aistats21} which is based on solving multiple instances of the optimization problem $(*)$.


\end{document}